\documentclass[preprint,11pt]{elsarticle}
\usepackage{amssymb}
\usepackage{bm}
\usepackage{amsmath}
\usepackage{url}
\usepackage[left=1in,right=1in,top=1in,bottom=1in]{geometry}
\usepackage{array}
\usepackage{multirow}
\usepackage{booktabs}
\usepackage{amsthm}
\newtheorem{theorem}{Theorem}
\newtheorem{lemma}{Lemma}

\newtheorem{assumption}{Assumption}
\usepackage{setspace}
\usepackage{algorithm}
\usepackage{algpseudocode}
\usepackage{subcaption}
\usepackage{graphicx}
\usepackage{diagbox}  
\usepackage{booktabs} 
\usepackage{makecell} 

\begin{document}

\begin{frontmatter}

\title{Generative AI Models for Learning Flow Maps of Stochastic Dynamical Systems in Bounded Domains}

\author[1]{M.~Yang} 
\author[2]{Y.~Liu} 
\author[3]{D.~del-Castillo-Negrete}
\author[4]{Y.~Cao}
\author[5]{G.~Zhang}
\affiliation[1]{organization={Fusion Energy Division, Oak Ridge National Laboratory}, 
            city={Oak Ridge},
            state={Tennessee}}
\affiliation[2]{organization={Department of Mathematical Sciences, Middle Tennessee State University}, 
            city={Murfreesboro},
            state={Tennessee}}
\affiliation[3]{organization={Institute for Fusion Studies, Dept of Physics, University of Texas at Austin}, 
            city={Austin},
            state={Texas}}
\affiliation[4]{organization={Department of Mathematics, Auburn University}, 
            city={Auburn},
            state={Alabama}}
\affiliation[5]{organization={Computer Science and Mathematics Division, Oak Ridge National Laboratory}, 
            city={Oak Ridge},
            state={Tennessee}}

\begin{abstract}
Simulating stochastic differential equations (SDEs) in bounded domains,  presents significant computational challenges due to particle exit phenomena, which requires accurate modeling of interior stochastic dynamics and boundary interactions. 
Despite the success of machine learning-based methods in learning SDEs, existing learning methods are not applicable to SDEs in bounded domains because they cannot accurately capture the particle exit dynamics. We present a unified hybrid data-driven approach that combines a conditional diffusion model with an exit prediction neural network to capture both interior stochastic dynamics and boundary exit phenomena. Our ML model consists of two major components: a neural network that learns exit probabilities using binary cross-entropy loss with rigorous convergence guarantees, and a training-free diffusion model that generates state transitions for non-exiting particles using closed-form score functions. The two components are integrated through a probabilistic sampling algorithm that determines particle exit at each time step and generates appropriate state transitions. The performance of the proposed approach is demonstrated via three test cases: a one-dimensional simplified problem for theoretical verification, a two-dimensional advection-diffusion problem in a bounded domain, and a three-dimensional problem of interest to magnetically confined fusion plasmas.
\end{abstract}


\tnotetext[fn1]{{\bf Notice}:  This manuscript has been authored by UT-Battelle, LLC, under contract DE-AC05-00OR22725 with the US Department of Energy (DOE). The US government retains and the publisher, by accepting the article for publication, acknowledges that the US government retains a nonexclusive, paid-up, irrevocable, worldwide license to publish or reproduce the published form of this manuscript, or allow others to do so, for US government purposes. DOE will provide public access to these results of federally sponsored research in accordance with the DOE Public Access Plan.}
\end{frontmatter}

\section{Introduction}
\label{sec:intro}
Stochastic differential equations (SDEs) in bounded domains constitute a fundamental mathematical framework for modeling complex dynamical systems where randomness and boundary conditions play an essential role in the underlying physics \cite{sobczyk2013stochastic}. 
An example of particular interest is the modeling and simulation of the dynamics of charged particles in magnetically confined plasmas of interest to controlled nuclear fusion. In this case, SDEs are used to model particle dynamics under the influence of collisions, and the assessment of confinement requires an accurate and efficient implementation of boundary conditions. Other examples include pollutant transport in the atmosphere and the oceans. 
In these applications, the exit problem becomes a central concern as particles can leave the computational region, fundamentally altering the system dynamics and requiring accurate prediction of escape probabilities and first passage times \cite{talkner1987discrete, yang2021feynman}. 

The mathematical significance of SDEs lies in their ability to provide both a particle-based description through stochastic trajectories and a continuum description via the associated Fokker-Planck partial differential equation (PDE) \cite{kolobov2003fokker, peeters2008fokker}. However, numerical solutions face substantial challenges in both formulations, particularly when boundaries are present. For high-dimensional bounded problems, solving the Fokker-Planck PDE directly becomes computationally challenging due to the curse of dimensionality and complex boundary condition implementation \cite{mceneaney2007curse}. Alternatively, Monte Carlo simulation of the SDE, while naturally suited for high dimensions, suffers from slow convergence and additional challenges in exit probability estimation \cite{higham2001algorithmic, higham2013mean}. 


Recent advances in machine learning have introduced powerful approaches for learning unknown stochastic dynamical systems from observational data. References~\cite{chen2024learning,chen2024data,chen2024modeling} developed stochastic flow map learning (sFML) methods that decompose the stochastic flow map into two components: deterministic sub-maps using residual networks and stochastic sub-maps using generative models. Reference~\cite{qi2023data} proposed a statistical-stochastic surrogate modeling strategy that couples mean statistics with stochastic fluctuations using neural network closures, demonstrating effectiveness on chaotic systems with strong instabilities. Other notable approaches include physics-informed neural networks (PINNs) for SDEs \cite{yang2020physics, friedrich2011approaching, chen2021solving, yang2024pseudoreversible}, which solve for the probability density function of stochastic differential equations, and Gaussian process methods for stochastic system identification \cite{archambeau2007gaussian, opper2019variational}. However, these existing methods primarily focus on unbounded domains and do not adequately address the complexities introduced by particle escape and boundary interactions—phenomena that are fundamental to many physics applications.

To overcome the numerical challenges faced by particle methods in bounded domains, we develop a novel unified hybrid data-driven framework that combines a training-free conditional diffusion model with an escape prediction neural network for learning stochastic flow maps of particles that can exit the computational domain. To our knowledge, this is the first machine learning surrogate modeling approach specifically designed for SDE trajectory simulation in bounded domains where particle escape fundamentally alters system dynamics. Our approach addresses this challenge by decomposing the complex problem into two specialized components: interior stochastic dynamics and boundary escape phenomena.
The escape prediction component employs a fully connected neural network that learns the conditional probability of particle exit within a given time interval, given the current position within the domain. The network is trained using binary cross-entropy loss with exit indicator data, and we provide rigorous theoretical analysis proving convergence to the true exit probability as training data increases. For non-exit particle propagation, we leverage our previous work on training-free conditional diffusion models \cite{liu2024training} that derive closed-form exact score functions and use Monte Carlo estimation to approximate the score directly from trajectory data, eliminating computational overhead and training instabilities associated with neural network-based score function learning.

The two components are integrated through a probabilistic sampling algorithm: at each time step, the escape prediction network determines whether a particle exits the domain, and if not, the diffusion model generates the next state transition. This sequential approach enables direct trajectory simulation while properly handling discontinuities at domain boundaries. Unlike PINN-based approaches that compute exit times as final quantities through solving differential equations \cite{li2024deep, nguyen2019first}, our generative framework simulates actual particle trajectories within bounded domains, combining generative modeling for complex state transitions with classification for sharp boundary decisions. This unified framework uniquely handles both interior stochastic dynamics and boundary escape phenomena within a single approach, maintaining high accuracy for bounded domain problems where traditional flow map methods fail to capture essential exit dynamics.

The remainder of this paper is organized as follows. In Section \ref{sec:problemsetting}, we formulate the bounded domain SDE problem and establish the mathematical framework. Section \ref{sec:diffusion_model} presents our unified generative model, including the escape prediction network with convergence analysis and the training-free diffusion model for interior dynamics. In Section \ref{sec:numeric}, we demonstrate our framework's effectiveness through three numerical examples: a one-dimensional analytical case that allows verification of our method's accuracy and convergence properties, a two-dimensional stochastic advection-diffusion transport problem that validates the framework's capability in handling bounded domain dynamics with complex flow structures, and a three-dimensional runaway electron application that showcases the framework's ability to handle complex, high-dimensional plasma physics problems with practical significance. The analytical example validates our theoretical foundations, the 2D transport problem demonstrates the method's robustness across different boundary conditions, and the runaway electron model demonstrates the practical use and computational efficiency gains over traditional Monte Carlo methods.

\section{Problem setting}
\label{sec:problemsetting}
We consider the following $d$-dimensional autonomous SDE
\begin{equation}\label{eq:sde}
X_t = X_0 + \int_0^t a(X_s) ds
+ \int_0^t b(X_s) dW_s \;\; \text{with}\;\; X_0\in \mathcal{D} \subset \mathbb{R}^d,
\end{equation}
where $W_t:= (W_t^1, \ldots, W_t^m)^{\top}$ is $m$-dimensional standard Brownian motion, the drift term $a(X_t):  \mathbb{R}^d \rightarrow \mathbb{R}^d$ corresponds to the deterministic components of the dynamical system, the diffusion term $b(X_t):\mathbb{R}^d \rightarrow \mathbb{R}^d$ captures the stochastic effects, such as those arising from particle collisions in plasma physics. 
We assume $a$ and $b$ are globally Lipschitz in $x$ uniformly with respect to $t$, and $X_0$ is the initial position in an open bounded domain $\mathcal{D} \subset \mathbb{R}^d$. We introduce a uniform temporal mesh 
\begin{equation}\label{tmesh}
    \mathcal{T} := \{t_n : t_n = n\Delta t, n = 0, 1, \ldots, N_T \},
\end{equation}
where $\Delta t = T/N_T$, such that the SDE in Eq.~\eqref{eq:sde} can be rewritten as a conditional form, i.e., 
\begin{equation}\label{eq:flowmap}
X_{t_{n+1}}^{t_n,x} = x+\int_{t_n}^{t_{n+1}} a(X_s^{t_n,x}) ds + \int_{t_n}^{t_{n+1}} b(X_s^{t_n,x}) dW_s, 
\end{equation}
where $X_{t_{n+1}}^{t_n,x}$ is the solution of the SDE at $t_{n+1}$ under the condition that $X_{t_n} = x$.

Instead of learning the flow map defined by the SDE in Eq.~\eqref{eq:flowmap} over the entire unbounded domain $\mathbb{R}^d$, this work focuses on learning the flow map within a bounded domain $\mathcal{D}$. Specifically, when a particle—i.e., a realization of the stochastic process $X_t$—first exits the domain $\mathcal{D}$, it is considered ``killed'', and its trajectory is terminated. Owing to the stochastic nature of the process, the time at which this first exit occurs is a random variable, commonly referred to as the first exit time, i.e.,
\begin{equation}\label{exit_time}
\theta_{n,x}:= \inf\{ t>t_n \,|\, X_{t_n} = x \in \mathcal{D}, \,\,X_t^{t_n,x} \not\in \mathcal{D}\},
\end{equation}
which is a function of the current time instant $t_n$ and the current state value $X_{t_n} = x$. Applying the exit time to the SDE in Eq.~\eqref{eq:flowmap}, 
\begin{equation}\label{eq:flowmap_exit}
X_{t_{n+1}\, \wedge\, \theta_{n,x}}^{t_n,x} = x+\int_{t_n}^{t_{n+1}\, \wedge\, \theta_{n,x}} a(X_s^{t_n,x}) ds + \int_{t_n}^{t_{n+1}\, \wedge\, \theta_{n,x}} b(X_s^{t_n,x}) dW_s, 
\end{equation}
where $t_{n+1}\, \wedge\, \theta_{n,x} = \min(t_{n+1}, \theta_{n,x})$ indicates that the trajectory will be terminated at $\theta_{n,x}$ if the particle exits the domain $\mathcal{D}$. It is known that the SDE in Eq.~\eqref{eq:flowmap_exit} is the stochastic representation of convection-diffusion PDEs with Dirichlet boundary conditions \cite{yang2021feynman,zhang2017backward,yang2023probabilistic}. 

The observation data set of the SDE in Eq.~\eqref{eq:flowmap_exit} includes $H\geq 1$ trajectories of the state ${{X}}_{t}$ at discrete time instants on the mesh $\mathcal{T}$ defined in Eq.~\eqref{tmesh}, denoted by
\begin{equation} \label{eq:raw_trajectory}
     \left(X_{t_0}^{(i)}, \Gamma_0^{(i)}\right), \left(X_{t_1}^{(i)}, \Gamma_1^{(i)}\right), \cdots, 
     \left( X_{t_{L_i}}^{(i)}, \Gamma_{L_i}^{(i)}\right), \quad i=1,\cdots, H,
\end{equation}
where $\{t_0, t_1, \ldots, t_{L_i}\} \in \mathcal{T}$, $L_i$ denotes the {\it final step} ({\it stopping}) index of the $i$-th trajectory, where $L_i \leq N_T$. Specifically, $L_i = N_t$ means the $i$-th trajectory is terminated at the terminal time $T$ without hitting the domain boundary. $\Gamma_l^{(i)}$ is binary indicator defined by
\begin{equation}\label{eq:indicator}
    \Gamma_{l}^{(i)} := \left\{ 
    \begin{aligned}
        & 1, \;\; X^{(i)}_{t_{l+1}} \text{ exits the domain }\mathcal{D},\\
        & 0, \;\; X^{(i)}_{t_{l+1}} \text{ stays in the domain } \mathcal{D},\\
    \end{aligned}
        \right.
\end{equation}
indicating whether the $i$-th trajectory exits the domain $\mathcal{D}$. Since each trajectory is terminated upon its first exit from the domain $\mathcal{D}$, the {\it final step} index $L_i$ for the $H$ trajectories may vary. Also, the indicator $\Gamma_{l}^{(i)}$ equals  1 only at the $L_i$-th step. 
The trajectory data in Eq.~\eqref{eq:raw_trajectory} can be segmented and reorganized into data pairs suitable for characterizing the input-output relationship of the target stochastic flow map, i.e.,
\begin{equation}\label{eq:pairs}
    (x_m, \Delta x_m, \gamma_m) := \left(X_{t_l}^{(i)}, \;\;X_{t_{l+1}}^{(i)} - X_{t_{l}}^{(i)},\;\; \Gamma_{l}^{(i)}\right),
\end{equation}
where index $m (i,l)$ provides a global enumeration of all transition pairs. For the $l$-th step of the $i$-th trajectory, $m$ equals the total number of steps from all previous trajectories plus $l$, i.e.,
\begin{equation}
m = \left\{ 
\begin{aligned}
    & l, && \text{when } i = 1, \; 1 \leq l \leq L_1 \\
    & \sum_{j=1}^{i-1} L_j + l, && \text{when } i > 1, \; 1 \leq l \leq L_i
\end{aligned}
\right.
\end{equation}
which leads to a total of $M=\sum_{i=1}^H L_i$ adjacent data pairs. 
We denote the collection of the paired samples as the
observation data set for the flow map, i.e., 
\begin{equation}\label{eq:obs}
    \mathcal{S}_{\rm obs} := 
    \left\{ (x_m, \Delta x_m, \gamma_m) \, | \, m=1,\cdots,M \right\}.
\end{equation}

We intend to develop a generative AI model to learn the flow map defined by the SDE in Eq.~\eqref{eq:flowmap_exit}. 
The key challenge is how to handle the random exit time defined in Eq.~\eqref{exit_time} by using the indicator data $\gamma_m$ in $\mathcal{S}_{\rm obs}$. 
It is evident that the exit time influences the distribution of the state $X_t$, not only by introducing a discontinuity in its probability density function near the boundary of $\mathcal{D}$, but also by affecting the distribution within the domain $\mathcal{D}$. Existing methods \cite{liu2024training} that learn the flow map from the data pairs $\{(x_m, \Delta x_m)|m=1,\ldots,M\}$ do not taking into account the exit indicator. These methods implicitly assume that trajectories always remain in the domain $\mathcal{D}$, and therefore cannot determine whether or when a trajectory should be terminated. 
Our approach, introduced in Section~\ref{sec:diffusion_model}, addresses this challenge by decoupling exit detection from state propagation. This separation enables specialized handling of boundary conditions while preserving physical accuracy and reducing computational cost.

\section{Supervised learning of a generative model for the stochastic flow map} \label{sec:diffusion_model} 
We develop a supervised learning of a generative AI model to address the challenges of learning stochastic flow maps within bounded domains. Our generative model consists of two components. The first component, introduced in Section \ref{sec:BCescape}, is a neural network model that predicts the {\em exit probability} of each trajectory of $X_t$ at a given time instant and a spatial location. This component will be used to determine whether a trajectory should be terminated in the generation process. The second component, introduced in Section \ref{sec:score}, is to generate the next state of each trajectory for the non-exit particles. In Section \ref{sec:unified_framework}, we will describe how the two components will work together to accurately simulate particle trajectories in bounded domains while properly handling the discontinuities at domain boundaries.

\subsection{Supervised learning of the exit probability}\label{sec:BCescape}
This section focuses on developing the first component of our generative model, which determines whether a trajectory should be terminated during the generation process. Instead of approximating the exit time in Eq.~\eqref{exit_time}, we intend to learn the exit probability defined by
\begin{equation}\label{eq:exit_prob}
    \mathbb{P}_{\rm exit}(x) := \mathbb{P}\left\{ \theta_{n,x} - t_n<\Delta t \,|\, X_{t_n} = x \in \mathcal{D}\right\},
\end{equation}
which is the probability of the exit time $\theta_{n,x}$ is smaller than $t_{n+1}$ given that the state $X_{t_n}$ is in the domain $\mathcal{D}$ at the time instant $t_n$. We intend to train a fully connected neural network, defined by
\begin{equation}\label{eq:exit_nn}
    F_\eta(x) = \text{sigmoid} \left(A^{(J)} \, \phi\left( A^{(J-1)} \, \phi\left( \cdots \, \phi\left( A^{(1)} x + B^{(1)} \right) \cdots \right) + B^{(J-1)} \right) + B^{(J)} \right),
\end{equation}
where $A^{(j)}$ and $B^{(j)}$ are the weight and bias of the $j$-th layer, $\phi$ is the LeakyReLU activation function, and the sigmoid function is added to the output layer to ensure the output is of the neural network is within the range $[0,1]$. The subscript $\eta$ of 
$F_\eta(x)$ represents the concatenation of the weights and the biases of all the layers.

\subsubsection{The loss function for training $F_{\eta}(x)$}\label{sec:training}
The challenge of training the neural network \( F_\eta(x) \) to predict the exit probability lies in the absence of labeled data for supervised learning; specifically, we only observe the binary indicator defined in Eq.~\eqref{eq:obs}. However, it turns out that the indicator values \( \gamma_m \) from Eq.~\eqref{eq:obs} can be effectively used as labels in the binary cross-entropy (BCE) loss function. With this formulation, the output of \( F_\eta(x) \) converges to the true exit probability \( \mathbb{P}_{\rm exit}(x) \)
in Eq.~\eqref{eq:exit_prob} as the loss function approaches its stationary point. 

Specifically, the BCE loss used to train $F_\eta(x)$ is defined by
\begin{equation}\label{eq:BCE_loss}
\mathcal{L}_{\rm BCE}(\eta) := -\frac{1}{M} \sum_{m=1}^{M} \left[\gamma_m \log F_{\eta}\left(x_{m}\right) + (1-\gamma_m) \log (1-F_{\eta}(x_{m}))\right],
\end{equation}
where $x_m$ and $\gamma_m$ are the samples from the training dataset $\mathcal{S}_{\rm obs}$ in Eq.~\eqref{eq:obs}. 
To understand the convergence behavior of training process using the BCE loss in Eq.~\eqref{eq:BCE_loss}, we first study the decay of the loss function for a fixed location $x\in \mathcal{D}$. Assume we have a subset of the training dataset $\mathcal{S}_{\rm obs}$ defined by 
\begin{equation}\label{eq:S_x}
    \mathcal{S}_x := \left\{ (x_j, \Delta x_j, \gamma_j) \, | \, x_j = x, j=1,\cdots,J \right\},
\end{equation}
where the starting location $x_j$ is identical for all pairs. Substituting the samples in $\mathcal{S}_x$ into the loss function $\mathcal{L}_{\rm BCE}(\eta;x)$, we have
%
\begin{equation}\label{eq:loss_x}
    \mathcal{L}_{\rm BCE}(\eta;x) = -\frac{1}{J} \left( J_{\rm exit} \log F_\eta(x) + J_{\rm non-exit} \log (1-F_\eta(x)) \right),
\end{equation}
where $J_{\rm exit}$ is the number of samples with $\gamma_j = 1$ and $J_{\rm non-exit} = J-J_{\rm exit}$ is the number of samples with $\gamma_j=0$. 
The gradient of the loss function $\mathcal{L}_{\rm BCE}(\eta;x)$ w.r.t the parameter $\eta$ is
\begin{equation}
\nabla_\eta \mathcal{L}_{\rm BCE}(\eta;x) 
= -\frac{1}{J} \left( \frac{J_{\rm exit}}{F_\eta(x)} - \frac{J_{\rm non-exit}}{1-F_\eta(x)} \right) \nabla_\eta F_\eta(x).
\end{equation}
When minimizing the loss function using the gradient descent method to get to its stationary point $\nabla_\eta \mathcal{L}_{\rm BCE}(\eta;x) = 0$, we have
\begin{equation}
    \frac{J_{\rm exit}}{F_\eta(x)} - \frac{J_{\rm non-exit}}{1-F_\eta(x)} = 0, 
\end{equation}
which leads to 
\begin{equation}
    F_\eta(x) = \frac{J_{\rm exit}}{J_{\rm exit}+J_{\rm non-exit}} \rightarrow \mathbb{P}_{\rm exit}(x) \;\; \text{as}\;\; J   \rightarrow \infty.    
\end{equation}
This implies that the BCE loss can be used to train the neural network model $F_\eta(x)$ to approximate the exit probability $\mathbb{P}_{\rm exit}(x)$ at a specific location as the number of samples goes to infinity. On the other hand,
our actual training dataset $\mathcal{S}_{\rm obs}$ contains samples distributed throughout the domain $\mathcal{D}$ rather than concentrated at a single point $x$. The total loss function can be viewed as an expectation over all possible positions $x$ in the domain $\mathcal{D}$
\begin{equation}\label{eq:ee}
\mathcal{L}_{\rm BCE}(\eta) = \mathbb{E}_{x\sim\rho(x)}[\mathcal{L}_{\rm BCE}(\eta;x)],
\end{equation}
where $\rho(x)$ represents the distribution of sample points $\{x_m\}_{m=1}^M$ of the dataset $\mathcal{S}_{\rm obs}$ in Eq.~\eqref{eq:obs}. Under the assumption that our neural network $F_\eta$ has sufficient capacity to approximate any continuous function on $\mathcal{D}$, i.e., the universal approximation property \cite{pinkus1999approximation}, and given that our training samples provide adequate coverage of the domain $\mathcal{D}$, minimizing the total loss function $\mathcal{L}_{\rm BCE}(\eta)$ leads to $F_\eta(x)$ approximating $\mathbb{P}_{\rm exit}(x)$ for all $x \in \mathcal{D}$ simultaneously.

\subsubsection{Convergence analysis of training $F_{\eta}(x)$ using the BCE loss}

In real experiments, the idealized dataset $S_x$ in Eq.~\eqref{eq:S_x} is impractical since we cannot expect infinite samples starting at any specific location $x$. Instead, we work with the observation dataset $S_{\rm{obs}}$ defined in Eq.~\eqref{eq:obs}, from which we extract the training data for the exit probability network as
\begin{equation}
S_{\rm{train}} := \{(x_m, \gamma_m) | m = 1, \ldots, M\},
\end{equation}
where the spatial locations $\{x_m\}_{m=1}^M \subset \mathcal{D}$ are the starting points from trajectory segments in $S_{\rm{obs}}$. These locations form an empirical distribution over the bounded domain $\mathcal{D}$. For the convergence analysis, we model this empirical distribution by assuming that the training sample starting points $\{x_m\}_{m=1}^M$ are drawn from an underlying density $\rho(x)$ that characterizes the spatial coverage of our training data within $\mathcal{D}$. While the samples are not strictly i.i.d.~due to their origin from SDE trajectories, this assumption allows us to analyze the convergence behavior as the dataset size $M$ increases. In the following, we provide a convergence analysis that establishes uniform convergence in probability with respect to the randomness induced by sampling from $\rho(x)$, i.e.,
\begin{equation}
\sup_{x \in \mathcal{D}} |F_\eta(x) - P_{\rm{exit}}(x)| \overset{\rho}{\longrightarrow} 0 \quad \text{ as } M \to \infty,
\end{equation}
where the convergence is taken over all possible realizations of the training dataset $\{x_m\}_{m=1}^M$ sampled from $\rho(x)$.

\begin{assumption}\label{assum_1}
Let $\mathcal{D} \subset \mathbb{R}^d$ be a bounded domain, and $\mathbb{P}_{\rm{exit}}(x)$ be the true exit probability function. We assume:
\begin{itemize}
   \item {Regularity of exit probability:} $\mathbb{P}_{\rm exit}(x)$ is $L$-Lipschitz continuous on $\mathcal{D}$:
   \begin{equation}
   |\mathbb{P}_{\rm exit}(x) - \mathbb{P}_{\rm exit}(y)| \leq L\|x - y\|, \quad \forall x, y \in \mathcal{D}.
   \end{equation}
   
\item {Sample distribution:} The training sample starting points $\{x_m\}_{m=1}^M$ are i.i.d. samples drawn from a density function $\rho(x)$ that is continuous on $\mathcal{D}$ and satisfies $0 < \rho_{\min} \leq \rho(x) \leq \rho_{\max} < \infty$ for all $x \in \mathcal{D}$.
   
   \item {Neural network regularity:} We restrict $F_\eta$ to neural networks with sufficient approximation capability that are $L_F$-Lipschitz continuous:
   \begin{equation}\label{eq:lipF}
   |F_\eta(x) - F_\eta(y)| \leq L_F\|x - y\|, \quad \forall x, y \in \mathcal{D}.
   \end{equation}
\end{itemize}
\end{assumption}

Under Assumption \ref{assum_1}, our convergence analysis proceeds through two key results. Lemma \ref{lem:coverage} establishes that for any $x \in \mathcal{D}$, there exists a training sample $x_m \in \mathcal{S}_{\rm train}$ sufficiently close to $x$ as $M \to \infty$; Lemma \ref{lem:pointwise} demonstrates pointwise convergence of $F_\eta(x_m)$ to $\mathbb{P}_{\rm exit}(x_m)$ at sample locations. These results combine to yield our main convergence theorem.

\begin{lemma}\label{lem:coverage}
Define the covering radius $\delta_M = \sup_{x \in \mathcal{D}} \min_{m=1,\ldots,M} \|x - x_m\|$. For any $r > 0$, we have $\mathbb{P}(\delta_M \geq r) \to 0$ as $M \to \infty$.
\end{lemma}

\begin{proof}
For any fixed $x \in \mathcal{D}$, the probability that no sample falls within distance $r$ is
\begin{equation}
\mathbb{P}\left(\min_{m=1,\ldots,M} \|x - x_m\| \geq r\right) = \left(1 - \mathbb{P}(z \in B(x,r) \cap \mathcal{D})\right)^M,
\end{equation}
where $z$ is a random sample drawn from density $\rho$. Since $\rho(x) \geq \rho_{\min}$ and the ball $B(x,r) \cap \mathcal{D}$ has positive volume, we have
\begin{equation}
\mathbb{P}(z \in B(x,r) \cap \mathcal{D}) = \int_{B(x,r) \cap \mathcal{D}} \rho(y) \, dy \geq \rho_{\min} \text{Vol}(B(x,r) \cap \mathcal{D}) > 0.
\end{equation}
Therefore, for any fixed $x \in \mathcal{D}$, we have
\begin{equation}\label{eq:pointwise_bound}
\mathbb{P}\left(\min_{m=1,\ldots,M} \|x - x_m\| \geq r\right) \leq (1 - \delta)^M \to 0 \text{ as } M \to \infty,
\end{equation}
where $\delta = \rho_{\min} \text{Vol}(B(x,r) \cap \mathcal{D}) > 0$.

Since $\mathcal{D}$ is bounded, for any $\varepsilon > 0$, we can construct a finite $\varepsilon$-net $\mathcal{N}_\varepsilon = \{y_1, y_2, \ldots, y_K\} \subset \mathcal{D}$ such that
for every $x \in \mathcal{D}$, there exists some $y_k \in \mathcal{N}_\varepsilon$ with $\|x - y_k\| \leq \varepsilon$. The cardinality satisfies $K \leq C_d \left({\text{diam}(\mathcal{D})}/{\varepsilon}\right)^d$ for some dimension-dependent constant $C_d$, and the diameter of the domain $\mathcal{D}$ is defined as $\rm{diam}(\mathcal{D})=\sup_{x,y\in \mathcal{D}}\|x-y\|$.
Choose $\varepsilon = r/2$, then for the net $\mathcal{N}_{r/2}$, we have
\begin{equation}
\mathbb{P}\left(\max_{k=1,\ldots,K} \min_{m=1,\ldots,M} \|y_k - x_m\| \geq r/2\right) \leq \sum_{k=1}^K \mathbb{P}\left(\min_{m=1,\ldots,M} \|y_k - x_m\| \geq r/2\right).
\end{equation}
Applying the pointwise bound in Eq.~\eqref{eq:pointwise_bound} to each net point, we have
\begin{equation}\label{eq:step3}
\mathbb{P}\left(\max_{k=1,\ldots,K} \min_{m=1,\ldots,M} \|y_k - x_m\| \geq r/2\right) \leq K (1-\delta')^M\to 0 \text{ as } M \to \infty,
\end{equation}
where $\delta' = \rho_{\min} \text{Vol}(B(y_k,r/2) \cap \mathcal{D}) > 0$.

For any $x \in \mathcal{D}$, let $y_k$ be the closest net point, so $\|x - y_k\| \leq r/2$. If $\max_k\min_{m} \|y_k - x_m\| < r/2$, then there exists  sample $x_m$ with $\|y_k - x_m\| < r/2$, which implies
\begin{equation}
\|x - x_m\| \leq \|x - y_k\| + \|y_k - x_m\| < r/2 + r/2 = r.
\end{equation}
Therefore,
\begin{equation}
\left\{\max_{k} \min_{m} \|y_k - x_m\| < r/2\right\} \subseteq \left\{\sup_{x \in \mathcal{D}} \min_{m} \|x - x_m\| < r\right\}.
\end{equation}
Taking complements, we have
\begin{equation}\label{eq:step4}
\left\{\delta_M \geq r\right\} \subseteq \left\{\max_{k} \min_{m} \|y_k - x_m\| \geq r/2\right\}.
\end{equation}
Combining Eqs.~\eqref{eq:step3} and~\eqref{eq:step4}, we have
\begin{align}
\mathbb{P}(\delta_M \geq r) &\leq \mathbb{P}\left(\max_{k} \min_{m} \|y_k - x_m\| \geq r/2\right) \\
&\leq K (1-\delta')^M\to 0 \text{ as } M \to \infty.
\end{align}
The proof is completed.
\end{proof}

\begin{lemma}\label{lem:pointwise}
Assume the BCE loss $\mathcal{L}_{\rm{BCE}}(\eta)$ converges to its global minimum. For any sample point $x_m$ in the training set, as $M \to \infty$, we have
\begin{equation}
|F_\eta(x_m) - \mathbb{P}_{\rm exit}(x_m)| \overset{\rho}{\longrightarrow} 0.
\end{equation}
\end{lemma}
\begin{proof}
Since $F_\eta$ corresponds to the global minimum of $\mathcal{L}_{\rm BCE}(\eta)$, we have $\nabla_\eta \mathcal{L}_{\rm BCE}(\eta) = 0$. We analyze the implications of this global condition for the behavior near training sample $x_m$.
For a fixed sample point $x_m$, let $N_m(r)$ be the number of training samples within distance $r$ of $x_m$, i.e.,
\begin{equation}
N_m(r) = \sum_{j=1}^M \mathbf{1}_{\{\|x_j - x_m\| \leq r\}}.
\end{equation}
By the assumption on $\rho(x)$, we have $\mathbb{E}[N_m(r)] = M \mathbb{P}(z \in B(x_m, r) \cap \mathcal{D}) \geq M \rho_{\min} \text{Vol}(B(x_m, r) \cap \mathcal{D}) > 0$, so $N_m(r) \to \infty$ as $M \to \infty$ for any fixed $r > 0$.

Among these $N_m(r)$ nearby samples, let $N_m^{\rm exit}(r)$ be those with $\gamma_j = 1$. The key insight is that for samples $x_j$ with $\|x_j - x_m\| \leq r$, the exit probabilities $\mathbb{P}_{\rm exit}(x_j)$ are close to $\mathbb{P}_{\rm exit}(x_m)$ due to Lipschitz continuity
$|\mathbb{P}_{\rm exit}(x_j) - \mathbb{P}_{\rm exit}(x_m)| \leq L r$.
By the law of large numbers, as $M \to \infty$ (with $r$ fixed), we have
\begin{equation}
\frac{N_m^{\rm exit}(r)}{N_m(r)} \to \mathbb{E}[\gamma_j | x_j \in B(x_m, r)] = \mathbb{E}[\mathbb{P}_{\rm exit}(x_j) | x_j \in B(x_m, r)].
\end{equation}
As $r \to 0$, by continuity of $\mathbb{P}_{\rm exit}$ and $\rho$, we have
\begin{equation}
\mathbb{E}[\mathbb{P}_{\rm exit}(x_j) | x_j \in B(x_m, r)] \to \mathbb{P}_{\rm exit}(x_m).
\end{equation}

Now consider the contribution to the BCE loss from samples in $B(x_m, r)$
\begin{equation}
\mathcal{L}_{\rm BCE}^{(r)}(\eta, x_m) = -\frac{1}{N_m(r)} \sum_{\|x_j - x_m\| \leq r} \left[\gamma_j \log F_\eta(x_j) + (1-\gamma_j) \log(1-F_\eta(x_j))\right].
\end{equation}
Since $F_\eta$ is $L_F$-Lipschitz, for $\|x_j - x_m\| \leq r$, we have
\begin{equation}
|F_\eta(x_j) - F_\eta(x_m)| \leq L_F r.
\end{equation}
Taking the gradient of $\mathcal{L}_{\rm BCE}^{(r)}(\eta, x_m)$ with respect to $\eta$ and setting it to zero at the optimum
\begin{equation}
\nabla_\eta \mathcal{L}_{\rm BCE}^{(r)}(\eta, x_m) = -\frac{1}{N_m(r)} \sum_{\|x_j - x_m\| \leq r} \left(\frac{\gamma_j}{F_\eta(x_j)} - \frac{1-\gamma_j}{1-F_\eta(x_j)}\right) \nabla_\eta F_\eta(x_j) = 0.
\end{equation}
In the limit as $r \to 0$, using the Lipschitz continuity of $F_\eta$, this becomes
\begin{equation}
\left(\frac{N_m^{\rm exit}(r)}{F_\eta(x_m)} - \frac{N_m(r) - N_m^{\rm exit}(r)}{1-F_\eta(x_m)}\right) \nabla_\eta F_\eta(x_m) = 0.
\end{equation}
Assuming $\nabla_\eta F_\eta(x_m) \neq 0$, we have  
$F_\eta(x_m) = \frac{N_m^{\rm exit}(r)}{N_m(r)}$.
By the  law of large numbers for i.i.d. samples from density $\rho(x)$, as $M \to \infty$ (for fixed $r$), we have $\frac{N_m^{\rm exit}(r)}{N_m(r)} \to \mathbb{E}[\mathbb{P}_{\rm exit}(x_j) | x_j \in B(x_m, r)]$. Subsequently taking $r \to 0$, by continuity of $\mathbb{P}_{\rm exit}$, this converges to $\mathbb{P}_{\rm exit}(x_m)$. Therefore, $F_\eta(x_m) \to \mathbb{P}_{\rm exit}(x_m)$.
The proof is completed.
\end{proof}

\begin{theorem}\label{thm:convergence}
Under Assumption \ref{assum_1}, when the BCE loss $\mathcal{L}_{\rm{BCE}}(\eta)$ converges to its global minimum, as the training dataset size $M \to \infty$, we have
\begin{equation}
\sup_{x \in \mathcal{D}} |F_\eta(x) - \mathbb{P}_{\rm exit}(x)| \overset{\rho}{\longrightarrow} 0.
\end{equation}
\end{theorem}
\begin{proof}
For any $x \in \mathcal{D}$, let $x_m^*$ be the nearest training sample to $x$, so $\|x - x_m^*\| = \min_{j=1,\ldots,M} \|x - x_j\| \leq \delta_M$, where $\delta_M = \sup_{x \in \mathcal{D}} \min_{j=1,\ldots,M} \|x - x_j\|$ is the covering radius from Lemma \ref{lem:coverage}.

We decompose the approximation error as
\begin{align}
|F_\eta(x) - \mathbb{P}_{\rm exit}(x)| &\leq \underbrace{|F_\eta(x) - F_\eta(x_m^*)|}_{I_1} + \underbrace{|F_\eta(x_m^*) - \mathbb{P}_{\rm exit}(x_m^*)|}_{I_2} \\
&\quad + \underbrace{|\mathbb{P}_{\rm exit}(x_m^*) - \mathbb{P}_{\rm exit}(x)|}_{I_3}.
\end{align}
For term $I_1$, by the Lipschitz continuity of $F_\eta$ in Eq.~\eqref{eq:lipF}, we have
\begin{equation}
|F_\eta(x) - F_\eta(x_m^*)| \leq L_F \|x - x_m^*\| \leq L_F \delta_M.
\end{equation}
For term $I_3$, by the Lipschitz continuity of $\mathbb{P}_{\rm exit}$ from Assumption \ref{assum_1}, we have
\begin{equation}
|\mathbb{P}_{\rm exit}(x_m^*) - \mathbb{P}_{\rm exit}(x)| \leq L \|x - x_m^*\| \leq L \delta_M.
\end{equation}
For term $I_2$, by Lemma \ref{lem:pointwise}, we have $|F_\eta(x_m^*) - \mathbb{P}_{\rm exit}(x_m^*)| \overset{\rho}{\longrightarrow} 0$  as $M \to \infty$ for any training sample $x_m^*$.
Taking the supremum over $x \in \mathcal{D}$, we obtain
\begin{equation}
\sup_{x \in \mathcal{D}} |F_\eta(x) - \mathbb{P}_{\rm exit}(x)| \leq (L_F + L) \delta_M + \sup_{j=1,\ldots,M} |F_\eta(x_j) - \mathbb{P}_{\rm exit}(x_j)|.
\end{equation}
As $M \to \infty$, by Lemma \ref{lem:coverage}, $\delta_M \overset{\rho}{\longrightarrow} 0$, and by Lemma \ref{lem:pointwise}, $\sup_{j=1,\ldots,M} |F_\eta(x_j) - \mathbb{P}_{\rm exit}(x_j)| \to 0$ in probability. Therefore,
\begin{equation}
\sup_{x \in \mathcal{D}} |F_\eta(x) - \mathbb{P}_{\rm exit}(x)| \overset{\rho}{\longrightarrow} 0.
\end{equation}
The proof is completed.
\end{proof}

\subsection{Supervised training of a non-exit trajectory generator}\label{sec:score}
Having established the exit probability prediction model $F_{\eta}(x)$ that determines whether a trajectory exits the bounded domain $\mathcal{D}$, we now turn to the second component of our generative model, i.e., the state transition for particles that remain within the domain. { The objective} of this work is to build a conditional generative model, denoted by
$G_\xi({x},z)$, to approximate the flow map defined in Eq.~\eqref{eq:flowmap_exit} for non-exit trajectories, where $x$ is the current state of $X_t$ and $z$ is a sample from the standard normal distribution.
Traditional generative models face significant limitations for this task. Normalizing flows require reversible neural network architectures and expensive Jacobian determinant computations~\cite{kobyzev2020normalizing}. Standard diffusion models require solving reverse-time differential equations for each sample, making them computationally expensive~\cite{song2020score,nichol2021improved}.

Our approach leverages a training-free conditional diffusion model to generate labeled data, enabling supervised learning of a simple neural network $G_\xi(x,z)$ without architectural constraints. The key insight is that while diffusion models can transform standard normal distributions into complex target distributions, training neural networks to learn the required score function is expensive and requires solving differential equations for sampling.
Instead, we derive an analytical approximation of the score function computed directly from trajectory data. We define a forward process mapping state transitions to standard normal distributions, then use Monte Carlo estimation to approximate the score function from available data. Converting the stochastic reverse process to a deterministic ordinary differential equation provides smooth mappings from normal variables to state transitions, generating labeled training pairs.

With labeled data, we train a fully connected neural network $G_\xi(x,z)$ with multiple hidden layers and ReLU activation using standard mean squared error loss. Unlike normalizing flows that require reversible architectures, our conditional diffusion approach enables supervised training without reversibility constraints. This framework offers computational efficiency by eliminating score function training, architectural freedom in network design, sampling efficiency through direct generation, and training stability via supervised learning. The mathematical details are provided in our previous work~\cite{liu2024diffusion, liu2024training}.

The diffusion model provides an elegant approach to learning the distribution of state transitions. We define forward and reverse processes in an artificial flow domain $\tau \in [0,1]$. The forward process begins with the exact state transition and progressively adds noise until reaching a standard normal distribution:
\begin{equation}\label{eq:forward}
d{{Z}}_\tau^{{x}} = b(\tau){{Z}}_\tau^{{x}} d\tau + \sigma(\tau)dW_\tau \quad \text{with } {{Z}}_0^{{x}} = {X}_{t_{n+1}} - {X}_{t_n},
\end{equation}
where the superscript $^x$ indicates the condition $X_{t_n} = x$, and initial state $Z_0^x$ is the target transition variable. With appropriately defined coefficients 
$$b(\tau) = \frac{d \log \alpha_\tau}{d\tau} \quad\text{and} \quad  \sigma^2(\tau) = \frac{d\beta_\tau^2}{d\tau} - 2\frac{d \log \alpha_\tau}{d\tau}\beta_\tau^2,$$ 
where $\alpha_\tau = 1-\tau$ and $\beta_\tau^2 = \tau$, the conditional distribution evolves as a Gaussian $p_{Z_\tau^x|Z_0^x} \sim \mathcal{N} (\alpha_\tau Z_0^x, \beta_\tau^2 \mathbf{I}_d)$.
The corresponding reverse-time process reconstructs samples from the transition distribution:
\begin{equation}\label{eq:reverse}
d{{Z}}_\tau^{{x}} = \left[ b(\tau){{Z}}_\tau^{{x}} - S({Z}_\tau^{{x}}, \tau) \right] d\tau + \sigma(\tau)dB_\tau \quad \text{with } {Z}_1^{{x}} = {Z} \sim \mathcal{N}(0, \mathbf{I}_d),
\end{equation}
where $B_\tau$ is the reverse-time Brownian motion, and $S({z}_\tau^{{x}}, \tau) := \nabla_{{z}} \log p_{Z_\tau^{{x}}}({z}_\tau^{{x}})$ is the score function that guides the diffusion back toward the target transition distribution. The score function can be approximated in a training-free fashion using Monte Carlo estimators (see Section 3.2 in \cite{liu2024training}) at any spatiotemporal location, such that the reverse-time SDE in Eq.~\eqref{eq:reverse} can be solved numerically given the observation dataset $\mathcal{S}_{\rm obs}$ in Eq.~\eqref{eq:obs}. 

However, the stochasticity of the reverse-time SDE does not provide a smooth function relationship between the initial and terminal states. 
To generate labeled data for supervised learning, we convert the stochastic reverse-time SDE into a deterministic reverse-time ordinary differential equation (ODE)
\begin{equation}\label{eq:ode}
d{Z}_\tau^{{x}} = \left[b(\tau)Z_\tau^{{x}} - \frac{1}{2}\sigma^2(\tau)S({Z}_\tau^{{x}}, \tau)\right] d\tau \quad \text{with } {Z}_1^{{x}} = {Z} \sim \mathcal{N}(0, \mathbf{I}_d),
\end{equation}
which provides a deterministic flow from $\tau=1$ to $\tau=0$. This conversion preserves the marginal distributions while providing deterministic trajectories. For each sample $x_m$ in $\mathcal{S}_{\rm obs}$, we draw one sample, denoted by $z_m$, from the standard normal distribution and solve the reverse-time ODE using the explicit Euler scheme, which uses Monte Carlo estimation directly from trajectory data without neural network training \cite{liu2024training}.. The output of the ODE solver, denoted by $y_m$, is the labeled data paired with $x_m$. We denote the final labeled dataset by
\begin{equation}\label{eq:labeldata}
\mathcal{S}_{\rm label} := \{(x_m, z_m, y_m) \mid (x_m, \Delta x_m) \in \mathcal{S}_{\rm obs} \text{ and } z_m \sim \mathcal{N}(0, \mathbf{I}_d), m = 1, \ldots, M\}.
\end{equation}
Note that the size of $\mathcal{S}_{\rm label}$ may exceed that of the training set $\mathcal{S}_{\rm obs}$ in Eq.~\eqref{eq:obs}, since multiple trajectories can be generated for the same sample $x_m \in \mathcal{S}_{\rm obs}$ by drawing different realizations from the standard normal distribution. With the labeled dataset, the desired generative model $G_\xi(x,z)$ is trained in a supervised fashion using the MSE loss function, i.e., 
\begin{equation}\label{eq:diffLoss}
\mathcal{L}_{G}(\xi) := -\frac{1}{M} \sum_{m=1}^{M} \left(y_m - G_\xi(x_m,z_m)\right)^2.
\end{equation}

\subsection{Summary of the workflow}\label{sec:unified_framework}
In the training phase, the two components $F_\eta(x)$ and $G_\xi(x,z)$ are trained separately using the BCE loss and the MSE loss, respectively. In the generation phase, trajectories of the target SDE in Eq.~\eqref{eq:flowmap_exit} is generated using the procedure described in Algorithm 1. 
\begin{algorithm}
\caption{Generating trajectories of the target SDE}
\begin{spacing}{1.2}
\begin{algorithmic}[1]
\Require The neural network models $F_\eta(x)$ and $G_\xi(x,z)$; 
\State Draw a sample $x_{t_0} \in \mathcal{D}$ from the distribution of $X_{t_0}$;
\For{$n = 0$ to $N_T-1$}
\State Predict the exit probability $\mathbb{P}_{\rm exit}(x_{t_n})$ by evaluating $F_\eta(x_{t_n})$;
\State Draw a sample $\nu$ from the uniform distribution $\mathcal{U}(0,1)$;
\If{ $\nu \ge  F_\eta(x_{t_n})$}
\State Draw a sample $z$ from the standard normal distribution;
\State Generate the state $x_{t_{n+1}}$ by $x_{t_{n+1}} = x_{t_n} + G_\xi(x_{t_n},z)$;
\State Add $x_{t_{n+1}}$ to the trajectory;
\Else
\State $N_{\rm exit} = n$;
\State {\bf break};
\EndIf
\EndFor
\State Return the generated trajectory $\{x_{t_n}\}_{n=0}^{N_{\rm exit}}$, where $N_{\rm exit} \le N_T-1$.
\end{algorithmic}
\end{spacing}
\end{algorithm}

The unified framework offers several compelling advantages for simulating kinetic transport processes. First, it effectively decouples the interior dynamics (handled by the diffusion model) from the boundary interactions (managed by the escape prediction model), allowing each component to specialize in its respective domain. Second, it elegantly handles the discontinuities at domain boundaries without requiring explicit boundary condition implementations in the governing equations. Third, by combining generative modeling for complex state transitions with classification for sharp boundary decisions, we achieve a robust representation of the complete transport dynamics. A key strength of our framework is its ability to generalize beyond the training data. Since both models learn the underlying physics rather than memorizing specific trajectories, the framework can accurately predict particle behavior under varying initial conditions and over extended time horizons. This capability is particularly valuable for plasma transport simulations, where long-term predictions and robustness to different initial configurations are essential. The numerical examples in Section~\ref{sec:numeric} demonstrate this generalization capability across a range of scenarios.

\section{Numerical results}\label{sec:numeric}
In this section, we present comprehensive numerical results demonstrating the performance of our unified hybrid data-driven framework across three test cases of increasing complexity. We begin with a one-dimensional stochastic differential equation that allows exact verification of our method's accuracy and convergence properties. We then examine a two-dimensional stochastic transport problem to validate the framework's capability in handling bounded domain dynamics with complex flow structures. Finally, we demonstrate the framework's effectiveness on runaway electron generation in tokamak plasmas, extending traditional 2D momentum-pitch modeling to a full 3D setting that includes radial transport. We evaluate our approach by comparing it against high-fidelity Monte Carlo simulations across multiple metrics: escape probability prediction, phase-space distribution reconstruction, and runaway electron generation under varying initial conditions. All simulations were implemented in PyTorch with GPU acceleration. The source code is publicly available at \url{https://github.com/mlmathphy/Diffusion_Runaway}, and all numerical results shown here can be exactly reproduced using the repository.

\subsection{Verification of algorithm accuracy}
This section uses a one-dimensional stochastic differential equation driven by pure Brownian motion to illustrate the methodology and verify the accuracy of the proposed unified framework. The dynamics is given by, for $t\in [0,T]$,
\begin{equation}\label{eq:ex1}
    X_t = X_0 + W_t \quad \text{with}\,\, X_0 \in [0,L]\subset \mathbb{R},
\end{equation}
where $W_t$ is standard Brownian motion and the process is absorbed (killed) upon first exit from the interval $[0,L]$. This corresponds to a random walk with absorbing boundaries at $x = 0$ and $x = L$.

We implement a computational setup with domain size $L = 6$ and time horizon $T = 3$. The numerical simulations employ a Monte Carlo solver with time step $\Delta t = 5\cdot 10^{-4}$ to ensure high temporal resolution. For the observation dataset $\mathcal{S}_{\text{obs}}$ defined in Eq.~\eqref{eq:obs}, we record trajectory snapshots at sampling interval $\Delta T = 0.05$, corresponding to every 100 simulation steps. The exact analytical solution for the exit probability over time interval $\Delta T$ is given by
\begin{equation}
\mathbb{P}_{\rm exit}(x,\Delta T)= 1 - \sum_{k=0}^{\infty} \frac{4}{(2k+1)\pi} \sin\left( \frac{(2k+1)\pi x}{L} \right) \exp\left( -\frac{1}{2} \left( \frac{(2k+1)\pi}{L} \right)^2 \Delta T \right).
\label{eq:exit_probability}
\end{equation}

Figure~\ref{fig:ex1_convergence} demonstrates the convergence behavior of our neural network approximation for the exit probability $\mathbb{P}_{\rm exit}(x,\Delta T)$ as the number of training trajectories increases. The Kullback–Leibler (KL) divergence is defined as the relative entropy from the approximate exit probability (generated from $F_\eta$) $\mathbb{\hat{P}}_{\rm exit}$ to the exact exit probability $\mathbb{P}_{\rm exit}$
\begin{equation}\label{eq_kl}
  D_{\rm KL}(\mathbb{P}_{\rm exit} \,\| \,\mathbb{\hat{P}}_{\rm exit}) = \sum_{x\in X} \mathbb{P}_{\rm exit}(x)\log{\left(\frac{\mathbb{P}_{\rm exit}(x)}{\mathbb{\hat{P}}_{\rm exit}(x)}\right)},
\end{equation}
where $D_{\rm KL}$ is approximated over $10^4$ uniform samples within $[0,L]$. The results show that the KL divergence decreases consistently as the number of training trajectories increases from $10^4$ to $10^5$, demonstrating clear convergence of the neural network approximation. This empirical convergence validates Theorem \ref{thm:convergence}, which establishes that as the size of $\mathcal{S}_{\text{obs}}$ increases, the neural network $F_\eta(x)$ converges to the true exit probability $\mathbb{P}_{{\rm exit}}(x)$.
\begin{figure}[h!]
    \centering
\includegraphics[width=0.4\linewidth]{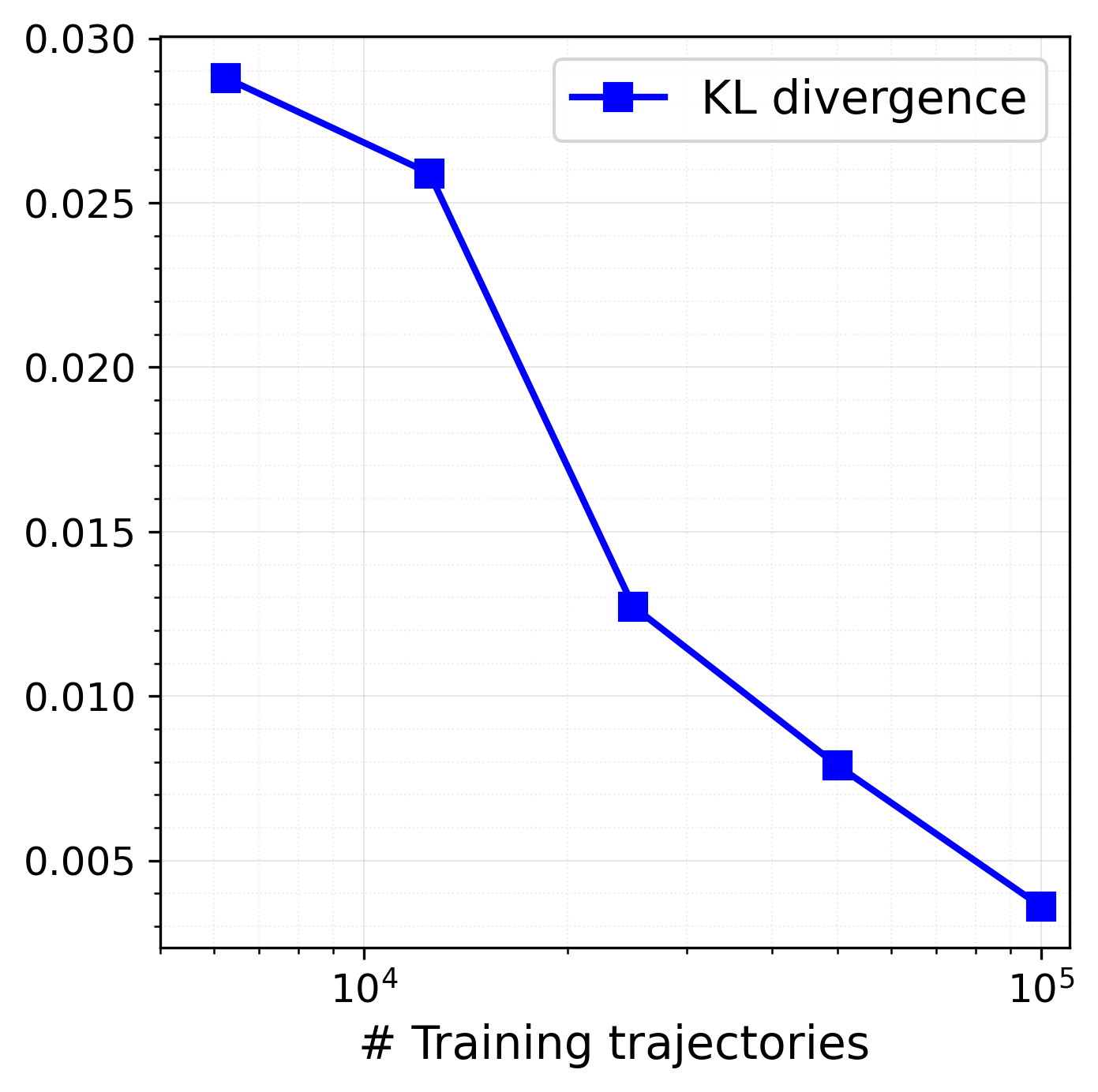}
    \caption{Convergence analysis for exit probability estimation. The plot shows the decay of KL divergence as a function of the number of training trajectories. The KL divergence measures the relative entropy between the neural network approximation and the analytical solution in Eq.~\eqref{eq:exit_probability}, computed over $10^4$ uniform samples within the domain $[0,L]$. This convergence behavior validates that as the size of $\mathcal{S}_{\text{obs}}$ increases, the neural network $F_\eta(x)$ converges to the true exit probability $\mathbb{P}_{{\rm exit}}(x)$.}
    \label{fig:ex1_convergence}
\end{figure}

Fig.~\ref{fig:ex1_trajectory} presents a comprehensive comparison of trajectory simulation methods for SDE $X_t$ in Eq.~\eqref{eq:ex1}, demonstrating the critical importance of proper exit dynamics modeling. The ground truth Monte Carlo simulation (left panel) establishes the reference behavior, showing particles starting from $x=1$ that gradually diffuse and exit through the absorbing boundaries at $x=0$ and $x=6$. Our unified method (second panel) achieves excellent agreement with the ground truth by combining the exit probability predictor $F_\eta(x)$ with the conditional diffusion model $G_\xi(x,z)$, successfully reproducing both the trajectory evolution patterns and the final particle distribution. In contrast, the naive approach of training a diffusion model on all trajectory data without escape/confined distinction (third panel) produces a fundamental artifact: an unphysical accumulation of particles near the boundary $x=0$, as evidenced by the spurious peak in the histogram. This occurs because the model learns to place particles near boundaries where they should have already exited, corrupting the physical representation. The fourth approach, training exclusively on confined trajectories, demonstrates the opposite pathology—the model becomes overly conservative and effectively prevents particle exit, leading to unrealistic over-population throughout the domain. These comparisons validate that exit phenomena require explicit probabilistic modeling rather than simple boundary conditions or selective training, which our unified framework successfully achieves through the principled separation of exit detection and state propagation.

\begin{figure}[htbp]
    \centering \includegraphics[width=0.95\linewidth]{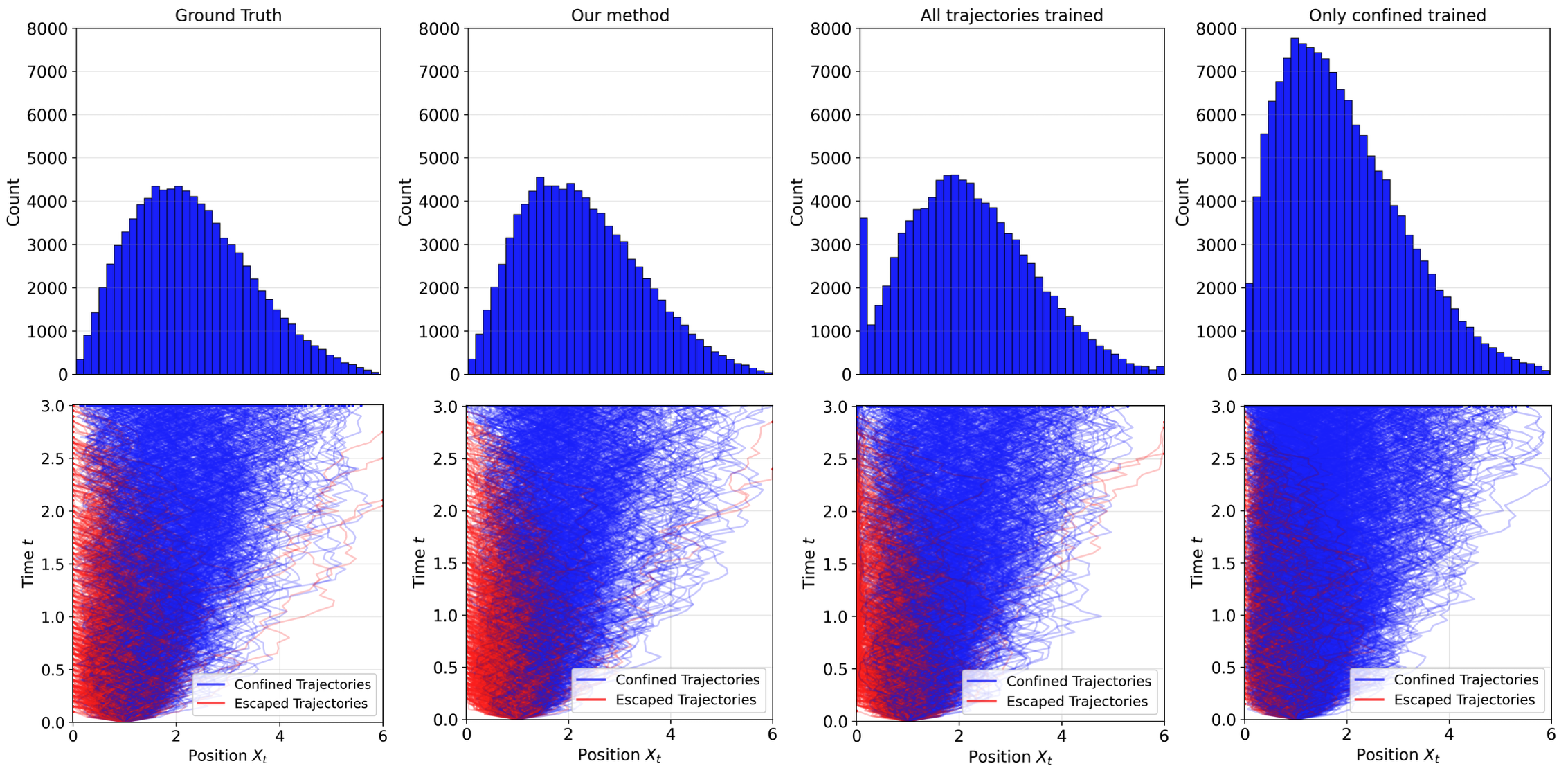}
    \caption{Comparison of trajectory simulation methods for $X_t$ in Eq.~\eqref{eq:ex1} absorbing boundaries at $x=0$ and $x=6$. Top row shows histograms of final particle positions at $T=3$ from $N_{\text{sample}}=200,000$ simulations, all starting from initial position $x=1$. Bottom row displays 500 representative trajectories colored by exit status (blue: confined, red: escaped). From left to right: (1) Ground truth Monte Carlo simulation, (2) Our unified method combining exit probability prediction and conditional diffusion, (3) Diffusion model trained on all trajectories without escape/confined distinction, showing artificial boundary accumulation, (4) Diffusion model trained only on confined trajectories, exhibiting unrealistic particle retention. The unified method accurately reproduces the ground truth behavior, while alternative approaches exhibit systematic biases that compromise physical accuracy.}
    \label{fig:ex1_trajectory}
\end{figure}

Table.~\ref{table:escape} provides quantitative validation of our method's accuracy in predicting particle confinement rate over time. The results demonstrate that our unified approach closely matches the ground truth confinement percentages at all time points, with deviations of around 1\% across all measured intervals. The method trained on all trajectories without escape/confined distinction shows moderate over-prediction of confinement rates, while the approach trained exclusively on confined trajectories severely overestimates confinement rates across all time points, predicting nearly double the true rate at later times. These quantitative results confirm that our unified framework accurately captures the temporal evolution of particle populations in bounded domains with absorbing boundaries.

\begin{table}[!ht]
\centering
{\fontsize{10}{9.5}\selectfont
\renewcommand{\arraystretch}{1.5}
\begin{tabular}{lcccc}
\toprule
\diagbox[width=2.cm]{Time}{Method} 
& Ground Truth 
& Our Method 
& \makecell[c] {All Trajectories\\ Trained} 
& \makecell[c]{Only Confined\\Trajectories Trained} \\
\midrule
\makecell[c]{T=1} & 65.08\%  & 66.23\% & 74.81\% & 89.14\% \\
\makecell[c]{T=2} & 49.94\%  & 50.63\% & 56.23\% & 79.26\% \\
\makecell[c]{T=3} & 43.80\%  & 44.22\% & 46.92\% & 72.62\% \\
\bottomrule
\end{tabular}}
\caption{Comparison of particle confinement rate (percentage of confined trajectories) across different simulation methods and time points, demonstrating the superior accuracy of our unified approach in capturing temporal evolution of particle populations.}
\label{table:escape}
\end{table}

\subsection{A 2D Stochastic Transport Problem}\label{sec:fluid}

To further validate our framework's capability in handling bounded domain problems, we consider a two-dimensional advection-diffusion transport system described by the SDE
\begin{equation}\label{ex3}
\left\{
\begin{aligned}
dx_1 &=  P_e\, v_{x_1} (x_1,x_2) \, dt +\, dW_1,\\ 
dx_2 &= P_e\, v_{x_2} (x_1,x_2) \,  dt + \, dW_2 \,   ,
\end{aligned}
\right.
\end{equation}
in the bounded domain $\mathcal{D} :=  [-\pi,\pi]\times[0,L]$ with a time-independent velocity field with components:
\begin{eqnarray}
v_{x_1} &= &- \pi \cos (\pi x_2)\sin (nx_1),  \\ 
v_{x_2} &=& n\sin (\pi x_2)\cos (nx_1),
\end{eqnarray}
where $dW_1$ and $dW_2$ are independent Wiener processes modeling the effect of diffusion. This cellular flow model represents transport in steady convective systems and poses challenges for traditional methods due to the combination of advective drift, stochastic diffusion, and domain boundaries.

The domain features periodic boundary conditions in $x_1$ and absorbing boundaries at $x_2 = 0$ and $x_2 = L$, where particles are removed upon contact. This creates a bounded domain problem where particle escape must be accurately modeled. We use the parameter values: $P_e=5$, $n=2$, $T_{\rm max} = 1$, and $L = 2$. The P\'eclet number $Pe = 5$ corresponds to a moderately advection-dominated regime where convective transport competes with diffusion. The numerical simulations employ a Monte Carlo solver with time step $\Delta t = 5\cdot 10^{-4}$ to ensure high temporal resolution. For the observation dataset $\mathcal{S}_{\text{obs}}$ defined in Eq.~\eqref{eq:obs}, we record trajectory snapshots at sampling interval $\Delta T = 0.05$.

Fig.~\ref{fig:ex2_exit} presents the exit probability $P_{\text{exit}}(x_1, x_2)$ with $\Delta T = 0.05$ across three different approaches. Our method (left panel) produces smooth, high-fidelity estimates, while the Monte Carlo simulations use 2000 samples for $n_x = 101$ and 1000 samples for $n_x = 31$ per grid point calculation. The Monte Carlo with $n_x = 101$ contour grid points (middle panel) exhibits statistical noise due to insufficient sampling relative to the fine contouring resolution, whereas the Monte Carlo with $n_x = 31$ contour grid points (right panel) shows smoother but less detailed contour curves. The exit probability pattern reveals a two-cell structure corresponding to the steady cellular flow with $n = 2$, where circulation centers around $x_1 = 0, \pm\pi$ create regions of low escape probability (dark blue) in the middle of the domain ($x_2 \approx 1$). High escape probability zones (red/orange) occur near the absorbing boundaries at $x_2 = 0, L$ and along the separatrices around $x_1 \approx \pm\pi/2$, where the flow structure facilitates particle transport toward the domain boundaries. The moderately advection-dominated regime with $P_e = 5$ creates a competition between convective trapping within circulation cells and diffusive escape, resulting in spatially heterogeneous patterns that reflect the underlying flow topology of this autonomous cellular flow model.

\begin{figure}[htbp]
    \centering \includegraphics[width=0.98\linewidth]{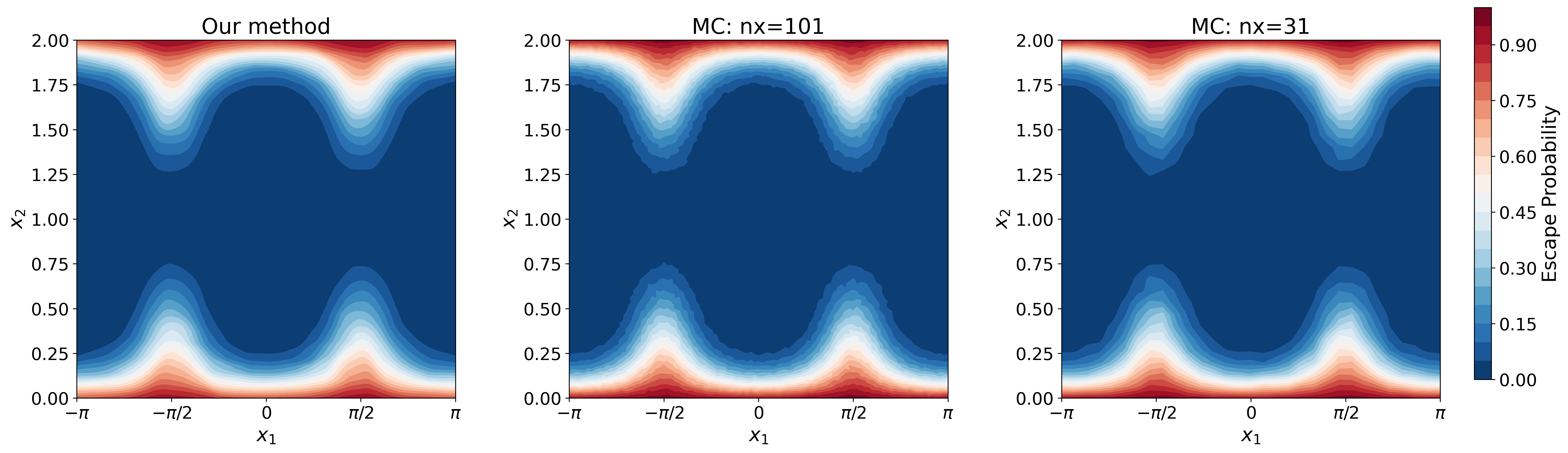}
     \caption{Exit probability $\mathbb{P}_{{\rm exit}}(x)$ with $\Delta T = 0.05$ for the 2D stochastic transport problem. Each panel shows $P_{\rm{exit}}(x_1, x_2)$ over the domain $[-\pi, \pi] \times [0, 2]$. Left: Our neural network prediction yielding smooth, high-fidelity results. Middle: Monte Carlo simulation with $n_x = 101$ contour grid points using 2000 samples per grid point. Right: Monte Carlo with $n_x = 31$ contour grid points using 1000 samples per grid point. The exit probability exhibits a two-cell structure with circulation centers at $x_1 = 0, \pm\pi$ showing low escape probability (dark blue) in the domain interior, and high escape probability zones (red/orange) near absorbing boundaries and separatrices around $x_1 \approx \pm\pi/2$.}
    \label{fig:ex2_exit}
\end{figure}

The spatial heterogeneity in exit probability demonstrates that different initial positions along $x_1$ exhibit varying sensitivity to exit phenomena, even though the absorbing boundaries are located in the $x_2$ direction. To demonstrate our model's capability in handling this sensitivity, we sample 50,000 trajectories from different starting locations to calculate the escape probability at $T_{\rm max}$ as a function of initial position $(x_1, x_2)$. We fix $x_2 = 1$ and vary $x_1$ over the range $[-\pi, \pi]$. Fig.~\ref{fig:ex2_exit_rate} shows the exit probability as a function of initial position $x_1$, revealing remarkable agreement between our method and the ground truth Monte Carlo simulation. The exit probability exhibits a clear periodic pattern reflecting the underlying cellular flow structure, with minimum values around the circulation centers at $x_1 = 0, \pm\pi$  and maximum values near the separatrices at $x_1 \approx \pm\pi/2$. This variation in exit probability across different initial positions highlights the strong influence of the flow topology on particle escape dynamics, where particles starting near separatrices experience enhanced transport toward boundaries while those near circulation centers remain more trapped within the cellular structure.
\begin{figure}[h!]
    \centering
\includegraphics[width=0.5\linewidth]{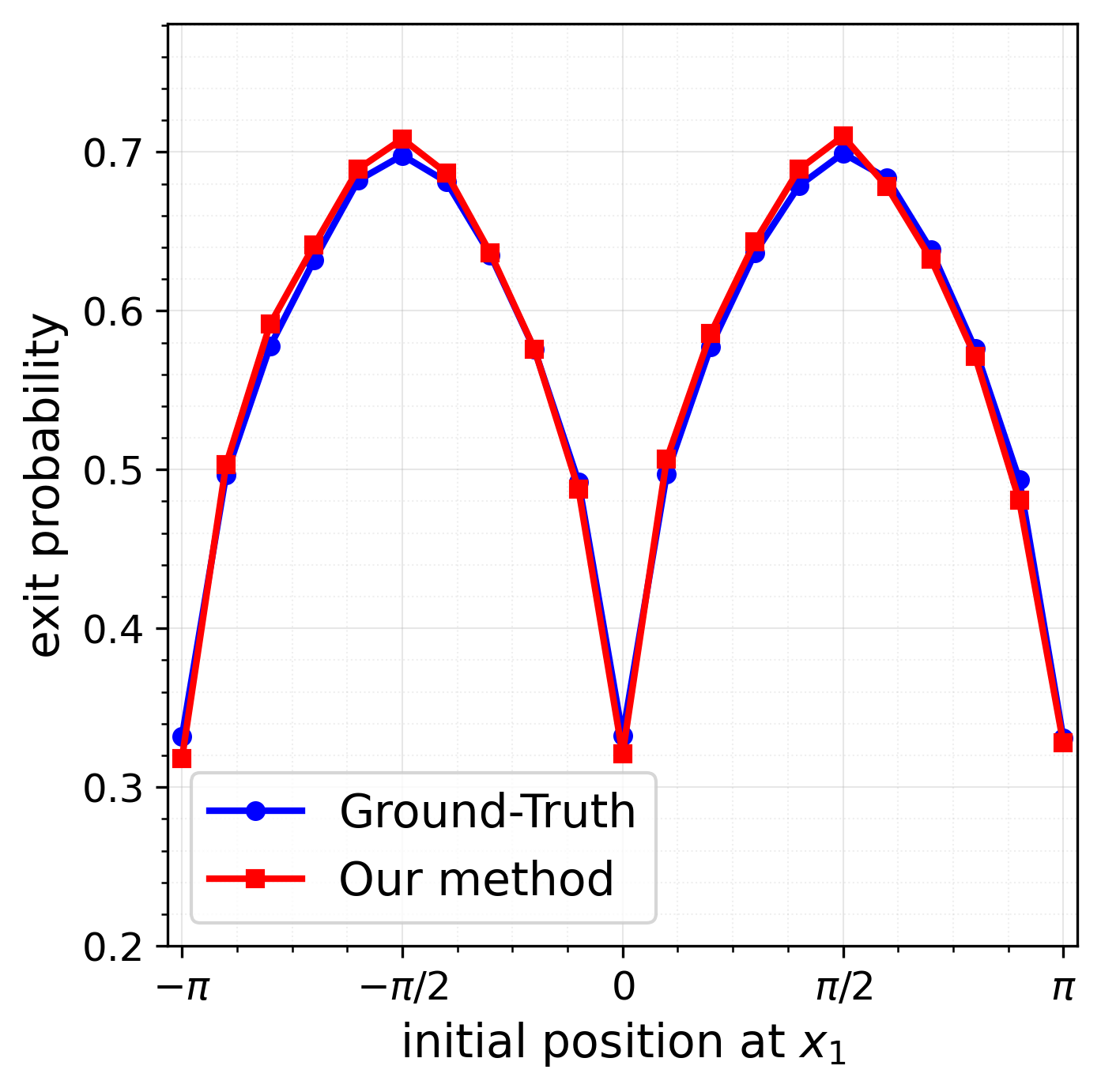}
 \caption{Exit probability at terminal time $T_{\rm max}$ as a function of initial position $x_1$ with fixed $x_2 = 1$. The plot shows exit probability calculated from $N_{\rm traj}=50,000$ trajectory samples for each initial position, i.e., $\text{exit probability} = N_{\rm escape}/N_{\rm traj}$. Our unified method (red squares) shows excellent agreement with ground truth Monte Carlo simulation (blue circles), accurately capturing the periodic variation in exit probability that reflects the underlying cellular flow structure with minima at circulation centers ($x_1 = 0, \pm\pi$) and maxima near separatrices ($x_1 \approx \pm\pi/2$).}
    \label{fig:ex2_exit_rate}
\end{figure}

Figs.~\ref{fig:ex2_marginal_distributions_2d_step_5_ini_3} and \ref{fig:ex2_marginal_distributions_2d_step_5_ini_2} present comprehensive comparisons of particle distributions at $T_{\rm max}$ to validate our framework's ability to reproduce the complete phase-space dynamics. Fig.~\ref{fig:ex2_marginal_distributions_2d_step_5_ini_3} shows results for particles initially sampled from $(x_1 = 0, x_2 = 1)$ with approximately 33,000 particles remaining in the domain, while Fig.~\ref{fig:ex2_marginal_distributions_2d_step_5_ini_2} presents results for initial conditions at $(x_1 = \pi/2, x_2 = 1)$ with approximately 15,000 particles remaining. Both figures display 1D marginal distributions and 2D contour plots comparing our method's predictions against Monte Carlo ground truth simulations. The results demonstrate excellent agreement across all distribution comparisons, successfully capturing the distinct particle evolution patterns that emerge from different initial positions within the cellular flow structure. For the $(x_1 = 0, x_2 = 1)$ initial condition, particles remain more concentrated around the circulation centers, while the $(x_1 = \pi/2, x_2 = 1)$ initial condition shows greater dispersion due to its proximity to the separatrix region. The strong quantitative agreement in both marginal and joint distributions, along with accurate particle count predictions, confirms that our framework effectively captures the complete stochastic dynamics while properly handling boundary escape phenomena for this autonomous 2D cellular flow transport problem.
\begin{figure}[h!]
    \centering \includegraphics[width=0.7\linewidth]{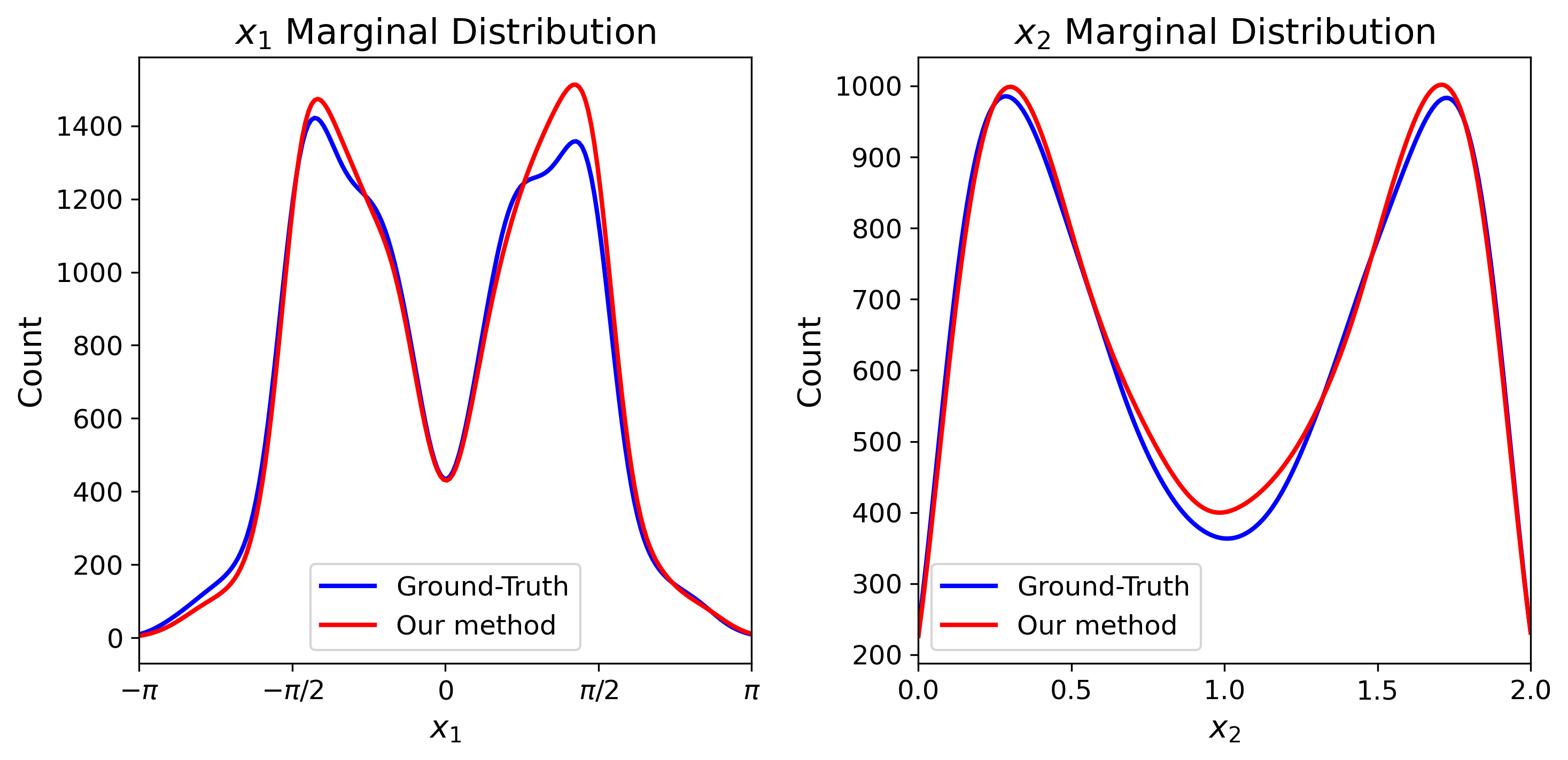}
\includegraphics[width=0.8\linewidth]{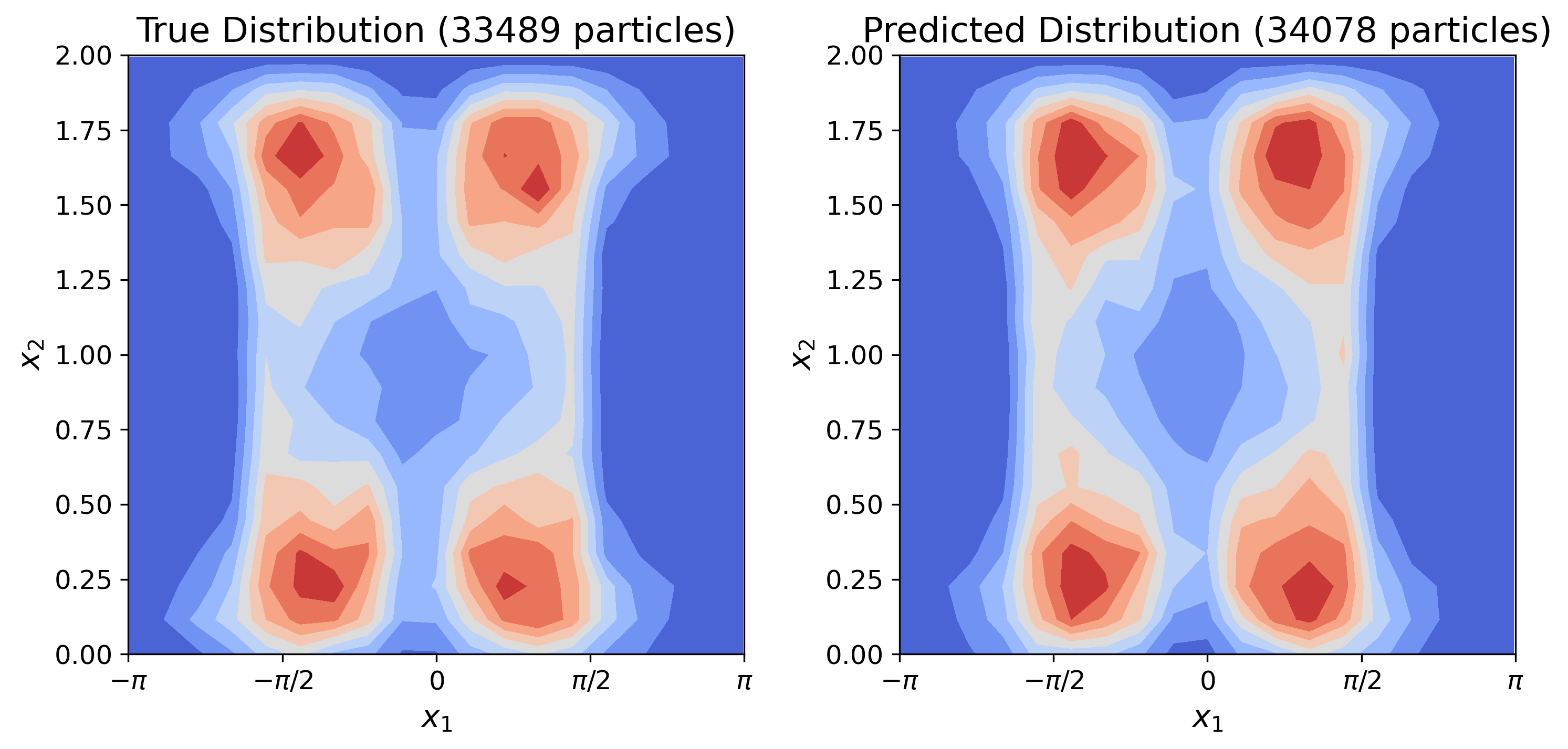}
\caption{Distribution comparison for particles initially sampled from $(x_1 = 0, x_2 = 1)$. Top panels show 1D marginal distributions for $x_1$ (left) and $x_2$ (right) comparing ground truth (blue) and our method (red). Bottom panels present 2D contour plots of the joint distribution: ground truth Monte Carlo simulation (left) with 33,489 particles and our method prediction (right) with 34,078 particles. The excellent agreement demonstrates our model's ability to accurately capture the complete phase-space evolution and particle count from this circulation center initial condition.}
\label{fig:ex2_marginal_distributions_2d_step_5_ini_3}
\end{figure}

\begin{figure}[h!]
    \centering \includegraphics[width=0.7\linewidth]{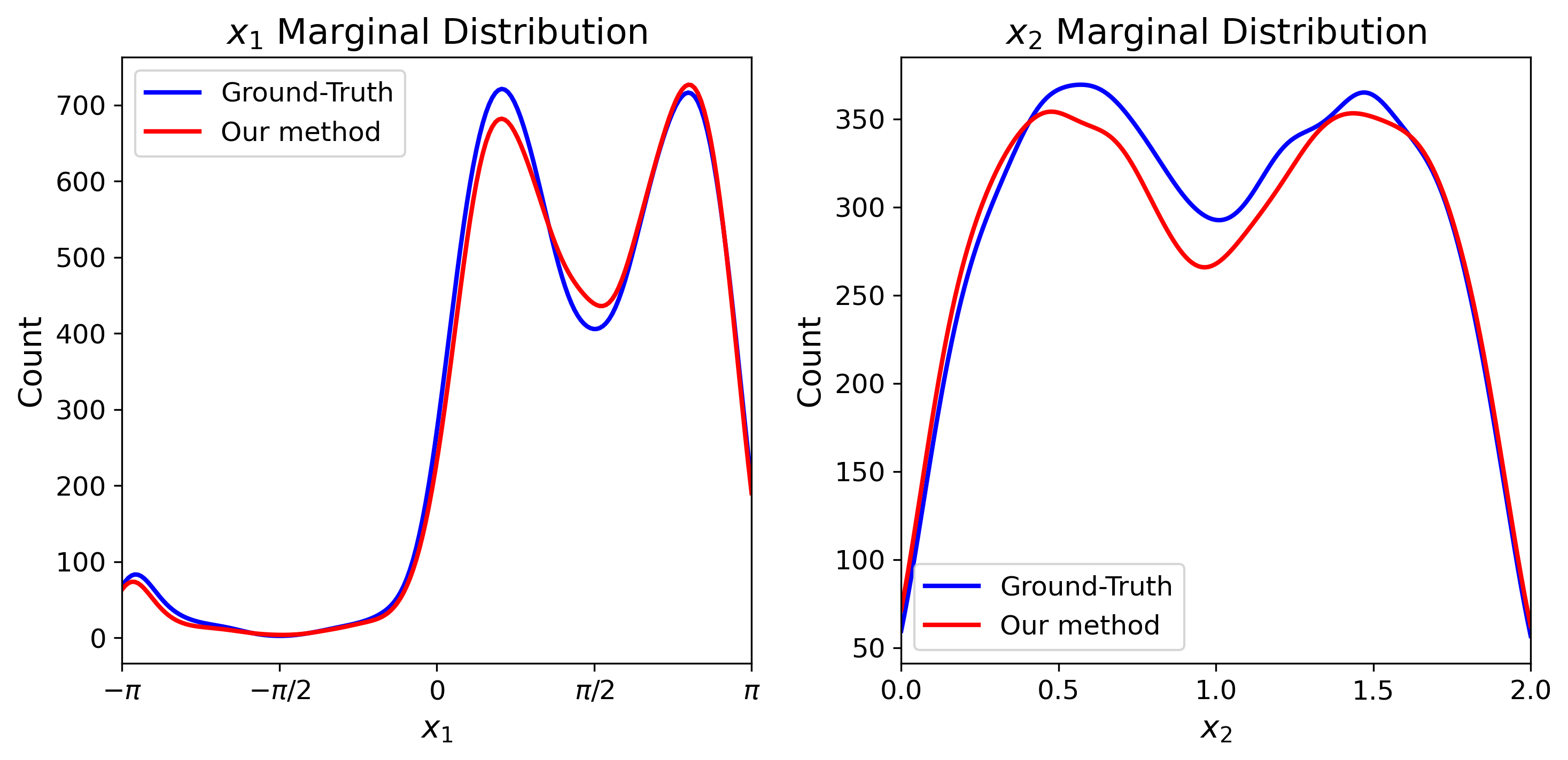}
\includegraphics[width=0.8\linewidth]{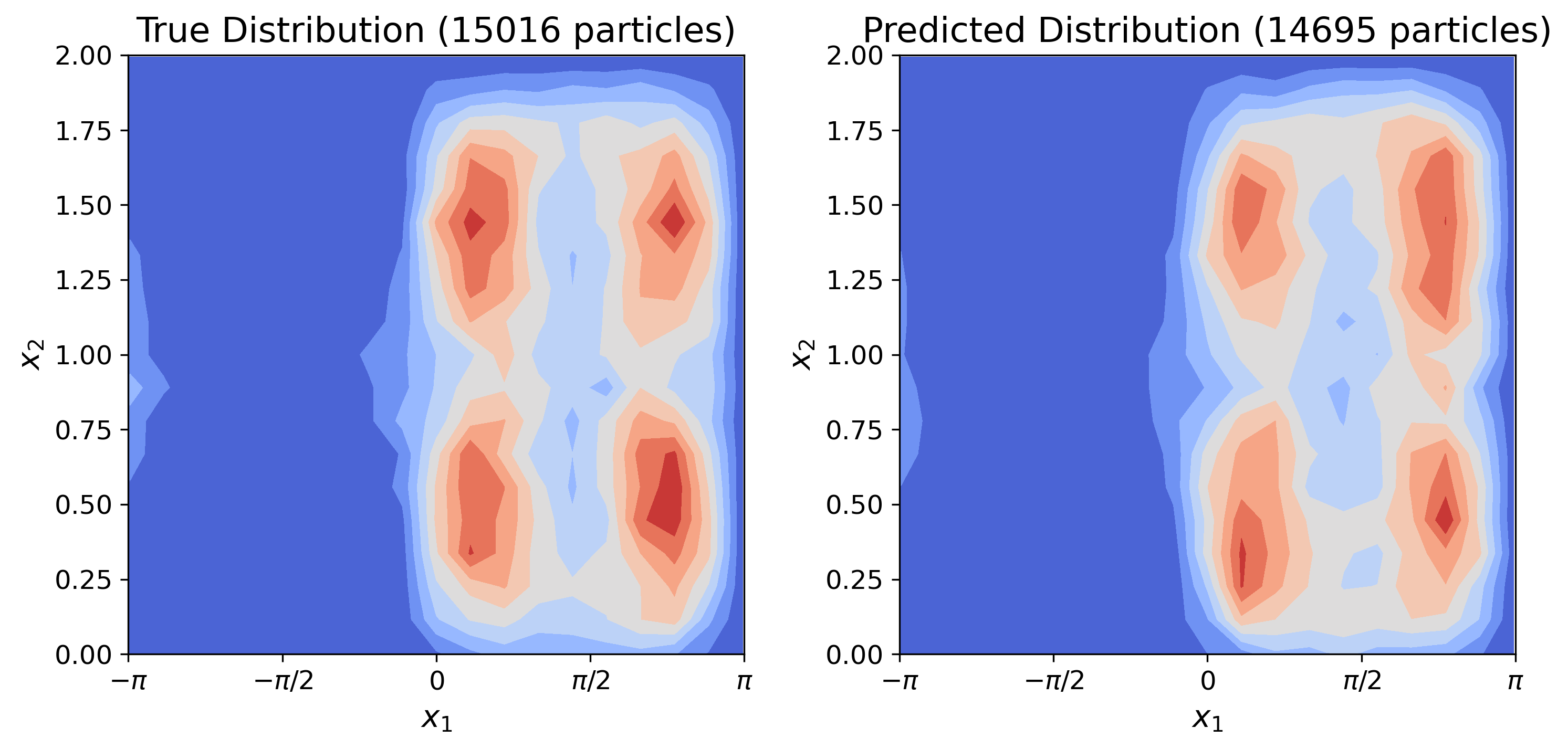}
\caption{Distribution comparison for particles initially sampled from $(x_1 = \pi/2, x_2 = 1)$. Top panels show 1D marginal distributions for $x_1$ (left) and $x_2$ (right) comparing ground truth (blue) and our method (red). Bottom panels present 2D contour plots of the joint distribution: ground truth Monte Carlo simulation (left) with 15,016 particles and our method prediction (right) with 14,695 particles. The strong agreement between simulations confirms our framework's capability to accurately reproduce distinct transport patterns and escape dynamics from this separatrix-proximate initial position.}
\label{fig:ex2_marginal_distributions_2d_step_5_ini_2}
\end{figure}

\subsection{Application to Runaway Electrons}\label{sec:RE_model}

In magnetically confined fusion plasmas, magnetic disruptions can give rise to the formation of high-energy electron beams, known as runaway electrons (RE). If not avoided or mitigated, escaping RE can produce significant damage the plasma facing components of fusion reactors leading to costly repairs and delays in the operation. Accordingly, a major goal of fusion research is the accurate computation of the rate of loss of confinement of RE. In this section we illustrate how the proposed generative AI hybrid framework can be used to study this problem.

\subsubsection{Runaway Electron Model and Simulation Setup}

The evolution of the RE is modeled using the following system of three-dimensional stochastic differential equations
\begin{equation}\label{eq_RE3d}
\left\{
\begin{aligned}
dp &=  \left [E \xi\, - \frac{\gamma p}{\tau}(1-\xi^2)-C_F(p) +\frac{1}{p^2}\frac{ \partial }{\partial p} \left(p^2 C_A\right) \right ] dt + \sqrt{2 C_A} \, dW_p,\\ 
d \xi &= \left[ \frac{E \left(1-\xi^2\right)}{p} + \frac{\xi (1-\xi^2)}{\tau \gamma}
-2 \xi \frac{C_B}{p^2} \right ] dt + \frac{\sqrt{2 C_B}}{p}  \,  \sqrt{1-\xi^2}\, dW_\xi \, ,\\
dr &= \frac{\partial D_r}{\partial r} dt + \sqrt{2D_r} dW_r,
\end{aligned}
\right.
\end{equation}
where $p$ is the relativistic momentum magnitude, $\xi = \cos \theta$ is the cosine of the pitch-angle, and $r$ is the minor radius. The system incorporates electric field acceleration $E$, synchrotron radiation damping, Coulomb collisions, i.e., momentum diffusion $C_A$, pitch angle scattering $C_B$, Coulomb drag $C_F$, and radial magnetic diffusion.

Time is normalized using the collisional time $1/\tilde{\nu}_{ee}$, where $\tilde{\nu}_{ee} = \frac{\tilde{n} e^4 \ln \tilde{\Lambda}}{4 \pi \epsilon_0^2 m_e^2 \tilde{v}_T^3}$ is the reference electron-electron thermal collision frequency. Momentum is normalized using $m \tilde{v}_T$, where $\tilde{v}_T = \sqrt{2 \tilde{T} / m}$ is the reference electron thermal velocity, and the electric field is normalized using $m \tilde{v}_T \tilde{\nu}_{ee} / e = \tilde{E}_D / 2$. The parameter $\tau = \tilde{\nu}_{ee} \tau_r$, where $\tau_r = 6 \pi \epsilon_0 m_e^3 c^3 / (e^4 B^2)$ is the synchrotron radiation time scale.
The collision coefficients are defined as:
\begin{equation}
C_A(p) = \bar{\nu}_{ee} \bar{v}_T^2 \frac{\psi(y)}{y}, \quad C_F(p) = 2\bar{\nu}_{ee} \bar{v}_T \psi(y), \quad C_B(p) = \frac{1}{2} \bar{\nu}_{ee} \bar{v}_T^2 \frac{1}{y} \left[ Z + \phi(y) - \psi(y) + \frac{y^2}{2}\delta^4 \right],
\end{equation}
where $\phi(y) = \frac{2}{\sqrt{\pi}}\int_0^y e^{-s^2}ds$, $\psi(y) = \frac{1}{2y^2}[\phi(y) - y\frac{d\phi}{dy}]$, and $y = \frac{1}{\bar{v}_T}\frac{p}{\gamma}$ with $\gamma = \sqrt{1 + (\tilde{\delta}p)^2}$, $\tilde{\delta} = \frac{\tilde{v}_T}{c} = \sqrt{\frac{2\tilde{T}}{mc^2}}$, $\delta = \frac{\hat{v}_T}{c} = \sqrt{\frac{2\hat{T}_f}{mc^2}}$, $\bar{v}_T = \sqrt{\frac{\hat{T}_f}{\tilde{T}}}$, and $\bar{\nu}_{ee} = \left(\frac{\tilde{T}}{\hat{T}_f}\right)^{3/2} \frac{\ln\hat{\Lambda}}{\ln\tilde{\Lambda}}$. Here $Z$ represents the ion effective charge.

The radial diffusivity is modeled as $D_r = D_0 F(r)G(p)$ with $D_0 = 0.01$ and
\begin{equation}
F(r) = \frac{1}{2}\left\{1 + \tanh\left[\frac{r-r_m}{L_D}\right]\right\}, \quad G(p) = e^{-(p/\Delta p)^2},
\end{equation}
where $r_m=0.5$ determines the boundary of magnetic stochasticity and $\Delta p$ determines the momentum scale beyond which electron orbits become insensitive to radial diffusion.
The electric field follows $E = E_0 \left[ \frac{\tilde{T}}{\hat{T}_f} \right]^{3/2}$ with typical model parameters $Z = 1$, $\tau = 6 \times 10^3$, $E_0 = 1/2000$, $\tilde{T} = 3$, $\hat{T}_f = 0.05$, and $mc^2 = 500$.

We solve the SDE system in Eq.~\eqref{eq_RE3d} using the Euler-Maruyama method with time step $\Delta t = 0.005$ over the domain $\mathcal{D} = \{(p,\xi,r) : p \in (0.5, 5), \xi \in (-1,1), r \in (0,1)\}$ up to terminal simulation time $T = 20$ and record states at intervals $\Delta t_{\rm obs} = 0.2$. A reflecting boundary condition is imposed at $r = 0$, while particles reaching $r \geq 1$ are considered escaped. 
 The initial conditions are derived from a family of Maxwell distributions 
\begin{equation}\label{eq_max}
f(p,\xi,r,t_0) = \frac{2p^2}{\pi^{1/2} p_0^3} e^{-(p/p_0)^2}, \quad \text{where} \,\,p_0 = \sqrt{\frac{\hat{T}_0}{\tilde{T}}},
\end{equation}
normalized such that $\int_{0}^1\int_{-1}^{1}\int_{p_{\min}}^{p_{\max}} f_0 (p,\xi,r) dp d\xi dr =1$ and parameterized by $\hat{T}_0$.

For improved numerical conditioning, we transform particle coordinates to physically meaningful components $p_{\parallel} = p\xi$ and $p_{\perp} = p\sqrt{1-\xi^2}$. The escape indicator is defined as
\begin{equation}
   \Gamma(p,\xi,r) :=\begin{cases}
       1, & r \geq 1 \quad \text{(escape)}, \\
       0, & r < 1 \quad \text{(non-escape)}.
   \end{cases}
\end{equation}
This setup generates datasets $\mathcal{D}_{\rm obs} = \{(x_i, \Delta{x}_i, \gamma_i)\}$ for flow map learning, where $x = (p_{\parallel}, p_{\perp}, r)$ represents the particle state.

\subsubsection{Escape Probability Estimation}

The escape prediction model employs a fully connected feedforward neural network with three hidden layers, each containing 256 neurons and LeakyReLU activation functions. Dropout with a rate of 0.2 is applied after each hidden layer to mitigate overfitting. The model is trained using the binary cross-entropy loss function, optimized with the Adam algorithm with a learning rate $l_r = 0.005$ and weight decay $10^{-5}$, trained for 100 epochs.



Fig.~\ref{fig:escape_cross_t20} shows cross-sectional comparisons of the escape probability $\mathbb{P}_{\rm exit}(p,\xi,r)$ with $\Delta t_{\rm obs}=0.2$ plotted over three 2D projections of the initial condition space: $(\theta, r)$ (top row), and $(r, p)$ (bottom row). The left column presents predictions from the escape prediction model, while the middle and right columns show Monte Carlo (MC) simulation results at high ($n=101$) and low ($n=31$) resolution, respectively. The model yields smooth, physically consistent estimates of escape probability, closely matching the converged high-resolution MC results while substantially outperforming the low-resolution MC simulations, which display considerable noise due to sample sparsity.
The increased exit probability at lower momenta observed in Fig.~\ref{fig:escape_cross_t20} results from the momentum-dependent radial diffusivity $D_r = D_0 F(r) G(p)$ with $G(p) = e^{-(p/\Delta p)^2}$. Since $G(p)$ increases as momentum decreases, low-energy particles experience stronger radial diffusion and higher escape rates compared to high-energy particles.

\begin{figure}
    \centering \includegraphics[width=0.95\linewidth]{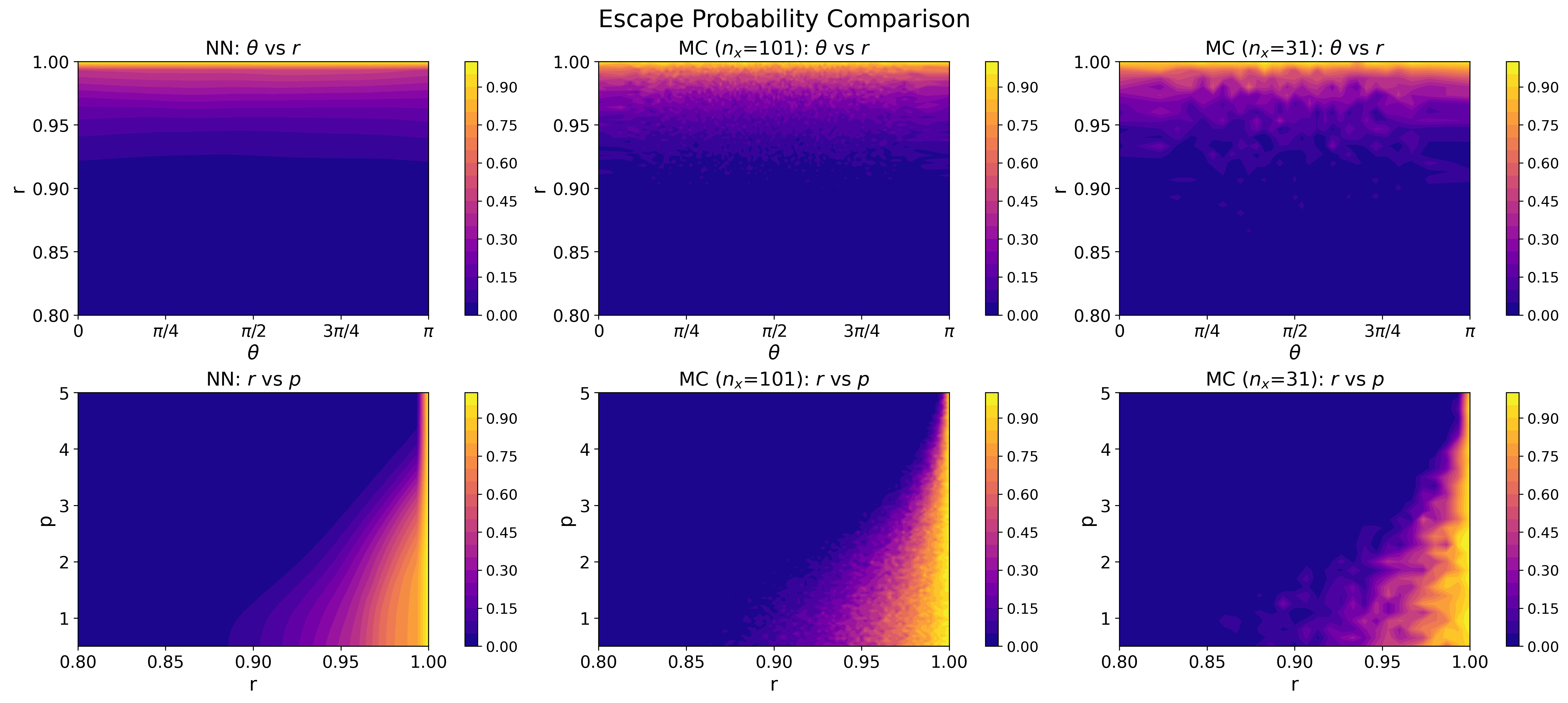}
   \caption{Escape probability cross-section comparison for runaway electron dynamics. Each row shows a projection of $P_{{\rm exit}}(p, \xi, r)$ onto 2D slices of the 3D phase space: $(p, \theta)$ (top) and $(r, p)$ (bottom). The escape prediction neural network (left column) yields smooth, high-fidelity results comparable to high-resolution Monte Carlo simulation with $n = 101$ grid points (middle column), and clearly outperforms low-resolution Monte Carlo with $n = 31$ grid points (right column), which suffers from statistical noise. The increased exit probability at lower momenta results from momentum-dependent radial diffusivity, where low-energy particles experience stronger radial transport.}
    \label{fig:escape_cross_t20}
\end{figure}

\subsection{Estimating the Probability Density of Runaway Electrons}\label{sec:numeric_RE}
In this section, we employ the proposed hybrid diffusion model discussed in Section~\ref{sec:diffusion_model} to compute the probability density function $f(p_t,\theta_t, r_t)$ related to the runaway electron generation model.
For the diffusion model, we implemented the training-free score estimation approach described in Section~\ref{sec:score} using $K=5000$ discretization steps for solving the reverse ODE.  We use 2048 nearest neighbors for the Monte Carlo estimation of the score function, striking a balance between accuracy and computational efficiency. The final generative model $G_\theta$ is a fully connected neural network with one hidden layer of 128 neurons, trained for 5000 epochs.

To evaluate the capability of the diffusion model in reproducing the phase-space behavior of runaway electrons, we compare the marginal probability density functions (PDFs) in $(p, \theta, r)$ against Monte Carlo (MC) simulation results for two initial Maxwellian distributions with  parameter $\hat{T}_0 = 4$ and $\hat{T}_0 = 10$, as shown in Fig.~\ref{fig:3D_marginal_PDF_maxwell}. Each panel displays marginal PDFs at time  $t=10$ and $t=30$, where $t=30$ lies in the model's prediction region beyond the training window. The diffusion model predictions agree closely with the MC benchmarks across all three variables and at both time points, confirming its ability to generalize beyond the training domain. The marginal PDFs show expected dependencies on initial temperature: higher $\hat{T}_0$ leads to broader momentum distributions and enhanced forward beaming in the pitch-angle direction.

\begin{figure}
    \centering
    \includegraphics[width=0.9\linewidth]{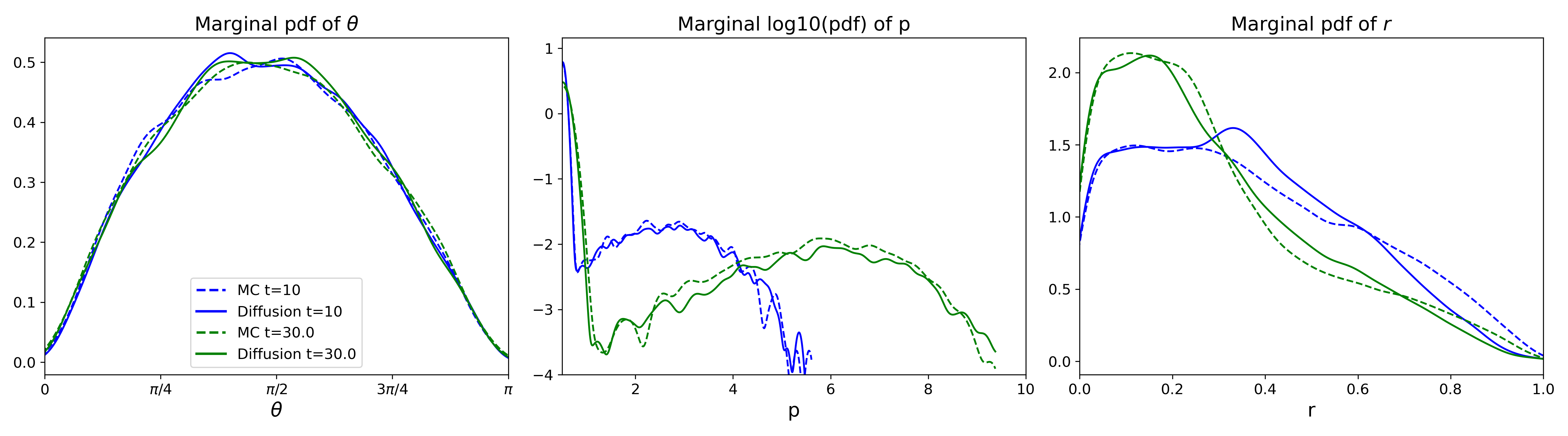}
    \includegraphics[width=0.9\linewidth]{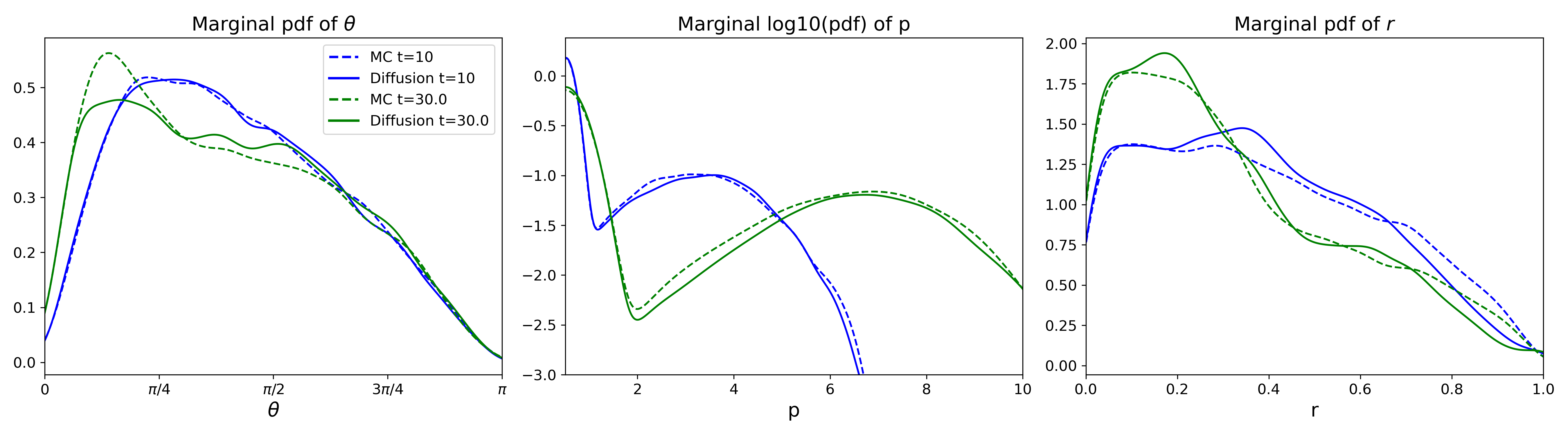}
    \caption{ Marginal PDFs of $\theta$, $p$, and $r$ at $t=10$ and $t=30$ for initial condition $\hat{T}_0 = 4$ (top), and $\hat{T}_0 = 10$ (bottom). The diffusion model (solid lines) shows strong agreement with Monte Carlo simulation (dashed lines) across all dimensions.}
    \label{fig:3D_marginal_PDF_maxwell}
\end{figure}

Fig.~\ref{fig:3D_PDF_theta_p_maxwell} presents 2D cross-sectional contour plots in the $(\theta, p)$ space at $t = 10$, $t = 20$, and $t = 30$. The comparison demonstrates strong agreement between Monte Carlo simulations (left panels) and our hybrid model (right panels) across all time frames. The temporal evolution clearly shows the formation and forward movement of the hot-tail population toward higher momentum values, with our model successfully reproducing this behavior and predicting evolution beyond the training window. Similar strong agreement between Monte Carlo simulations and our hybrid model is observed in the $(p, r)$ and $(r, \theta)$ cross-sectional views (not shown), confirming the model's accuracy across all coordinate planes and its ability to capture the full anisotropic dynamics of the particle evolution.
\begin{figure}[h!]
    \centering
    \includegraphics[width=0.7\linewidth]{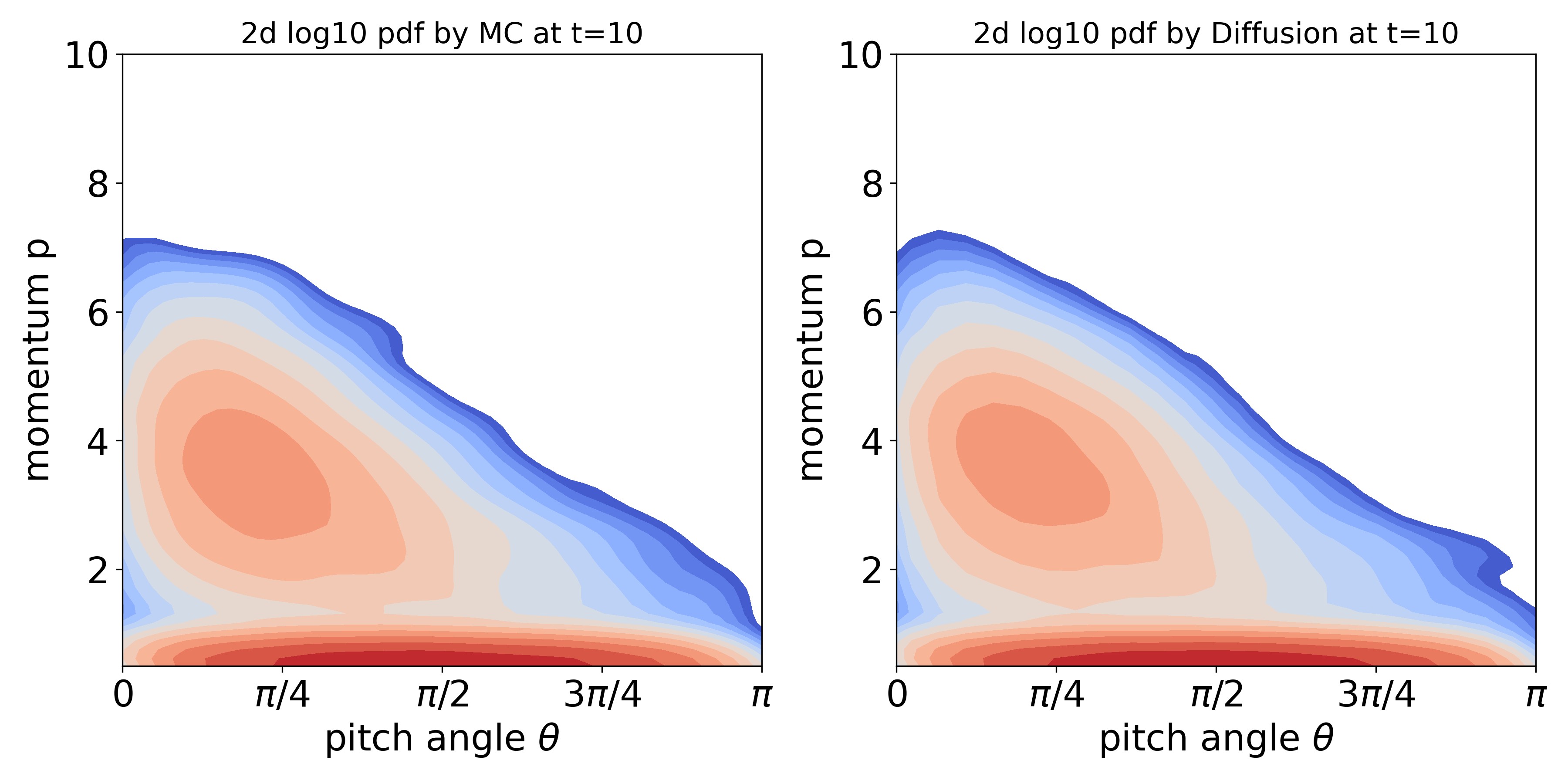}\\
    \includegraphics[width=0.7\linewidth]{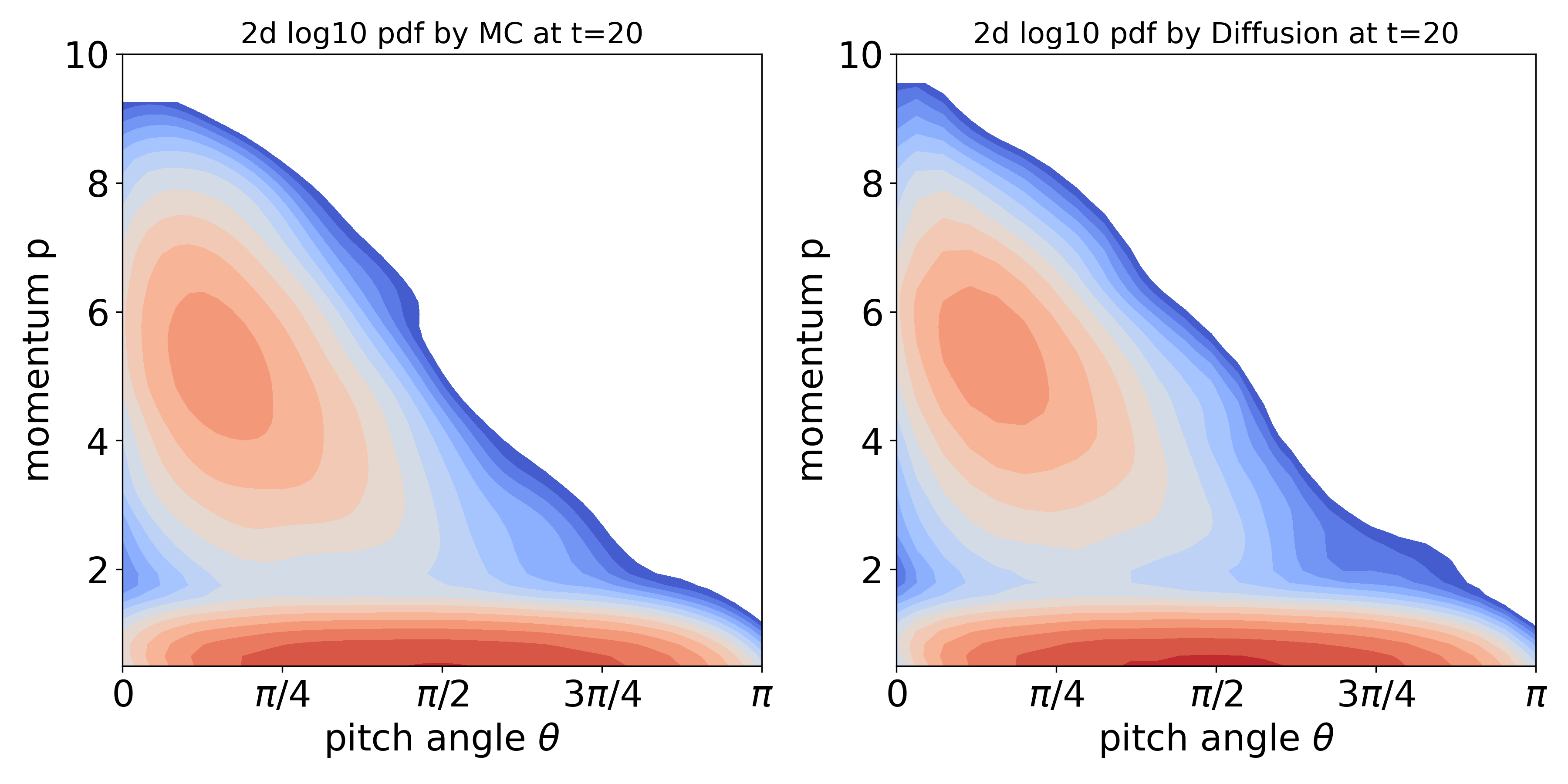}\\
    \includegraphics[width=0.7\linewidth]{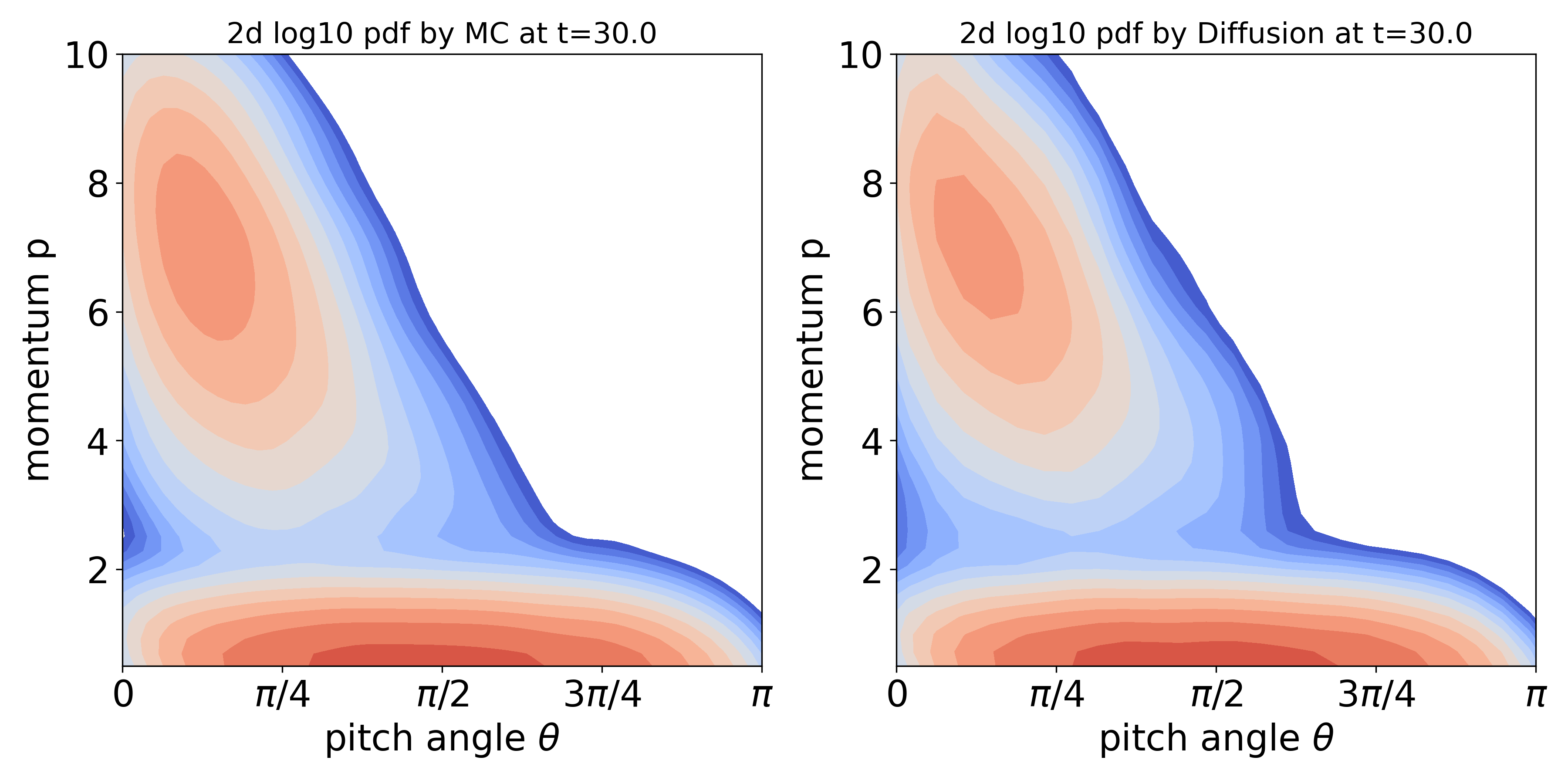}\\
    \caption{2D distribution ${\rm log}_{10}$-pdf $f(\theta,p, t)$ at $t=10$ (top), $t=20$ (middle) and $t=30$ (bottom) for initial condition $\hat{T}_0 = 10$. Left: Monte Carlo simulation. Right: Hybrid diffusion model. Note that $t=30$ is a prediction time state beyond the training window.}
    \label{fig:3D_PDF_theta_p_maxwell}
\end{figure}

To assess the accuracy of the PR-NF method in predicting the whole dynamics for varying initial conditions, we evaluate the following quantity of interest:
\begin{equation}\label{eq_pr}
n_{\rm RE} = \int_{0}^{1}\int_{-1}^{1}\int_{p^*}^{p_{\max}} f_{t} (p,\xi,r) \, dp \, d\xi \,dr,
\end{equation}
which represents the total fraction of runaway electrons generated by the hot-tail mechanism during the thermal quench. Here, $f_{t}$ denotes the electron distribution at time $t$, $p^* = 1.75$ is the threshold momentum for runaway electrons, and $p_{\max}$ is a numerical cutoff. We conduct this analysis over a set of Maxwellian initial conditions defined in Eq.~\eqref{eq_max}, with $\hat{T}_0$ ranging from 1 to 10.

The left panel in Fig.~\ref{fig:3d_PR} shows $n_{\rm RE}$ at $t = 5$ and $t = 20$ as a function of $\hat{T}_0$, highlighting the impact of initial thermal energy on runaway production. The right panel presents the temporal evolution of $n_{\rm RE}$ for two specific initial conditions, $\hat{T}_0 = 4$ and $\hat{T}_0 = 10$. These results demonstrate that the diffusion-based model combined with the escape prediction network achieves excellent agreement with Monte Carlo (MC) simulations. Moreover, the model accurately captures how low initial energy electrons thermalize more rapidly—resulting in lower $n_{\rm RE}$—while high energy electrons sustain runaway behavior. 
\begin{figure}[h!]
    \centering
    \includegraphics[width=0.4\linewidth]{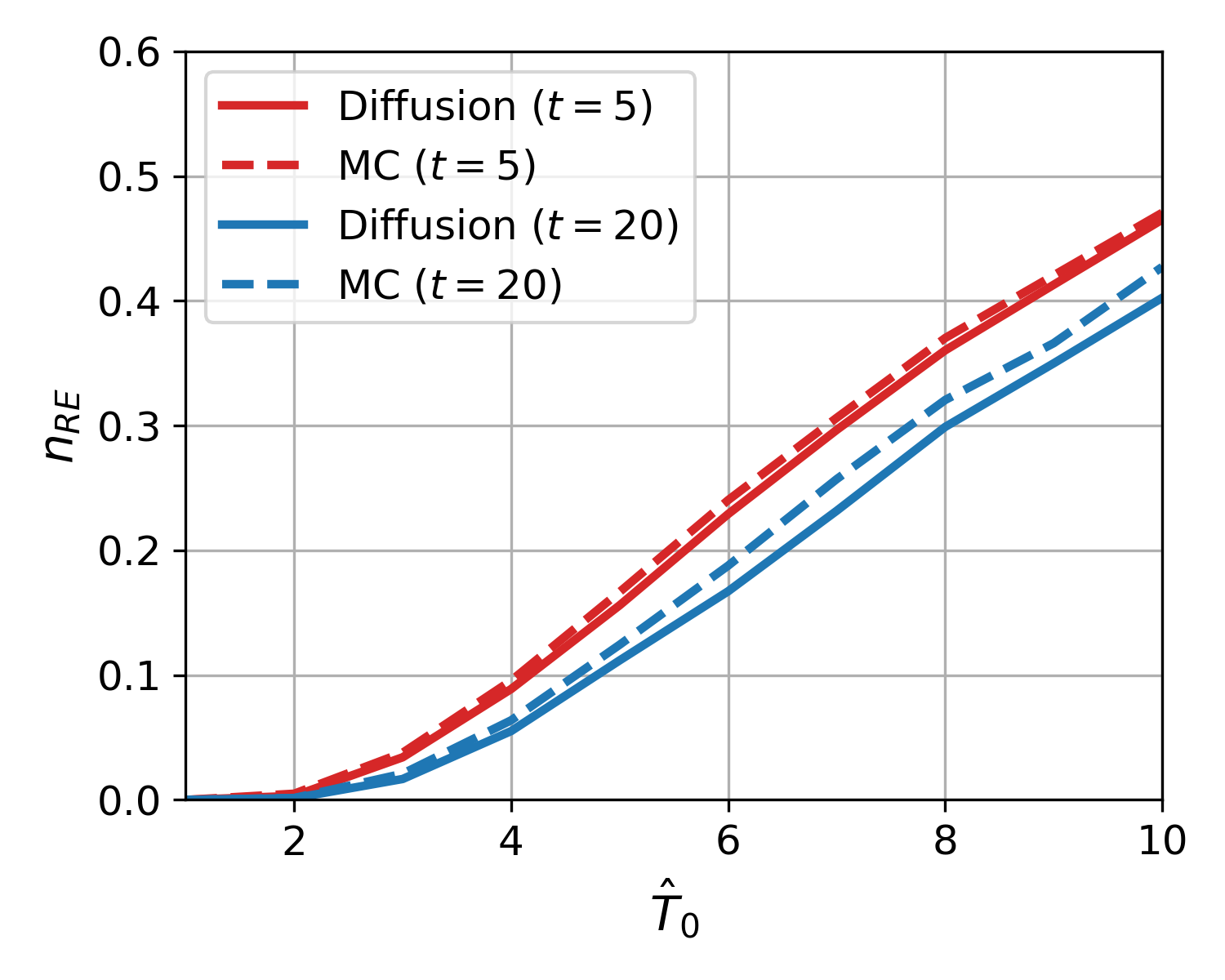}
    \includegraphics[width=0.4\linewidth]{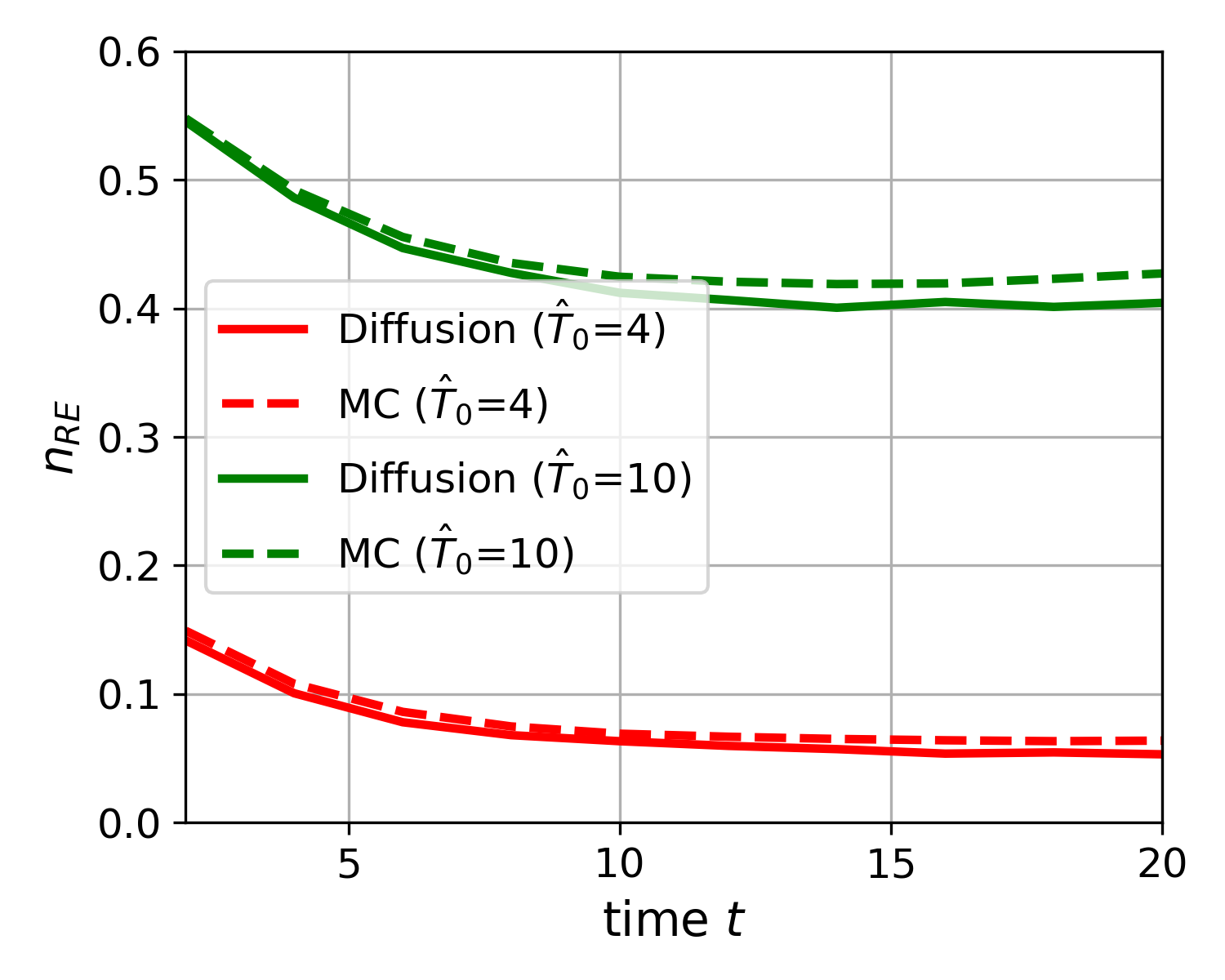}\\
\caption{Comparison of runaway electron production $n_{\rm RE}$ between the proposed diffusion model with escape prediction (solid lines) and Monte Carlo simulations (dashed lines). Left: $n_{\rm RE}$ at $t = 5$ and $t = 20$ for varying initial temperatures $\hat{T}_0$. Right: time evolution of $n_{\rm RE}$ for two representative cases, $\hat{T}_0 = 4$ and $\hat{T}_0 = 10$. The proposed method shows excellent agreement with MC results while enabling efficient prediction across different initial conditions.}
    \label{fig:3d_PR}
\end{figure}

Table~\ref{table:time} presents a computational performance comparison between the hybrid diffusion framework and traditional Monte Carlo methods. The hybrid diffusion approach requires an initial training cost but exhibits significantly different scaling behavior during evaluation. While Monte Carlo runtime increases proportionally with sample size following $O(N)$ scaling, the hybrid diffusion framework maintains nearly constant evaluation time across different sample sizes, achieving approximately 700× speedup at 200K samples. This constant-time scaling results from neural network architecture enabling parallel GPU processing, making the approach particularly advantageous for parameter space exploration and uncertainty quantification applications.

\begin{table}[!ht]
\centering
{\fontsize{8}{9.5}\selectfont
\renewcommand{\arraystretch}{1.5} 
\begin{tabular}{lccc|ccc}
\toprule
\multirow{2}{*}{\textbf{Models}} & \multicolumn{3}{c}{\textbf{Training Phase}} & \multicolumn{3}{c}{\textbf{Evaluation Phase}} \\
\cmidrule(lr){2-4} \cmidrule(lr){5-7}
 & Data labeling & Training $G(\cdot)$ & BC model & 50K samples & 100K samples & 200K samples \\
\midrule
Hybrid Diffusion & 344.13 (secs)  & 14.55 (secs)  &  39.15 (secs) & 39.05e-02 (secs) &  38.89e-02 (secs) &  40.46e-02 (secs)\\
Monte Carlo  &  \text{N/A} &  \text{N/A} &  \text{N/A} &  68.13 (secs) & 142.92 (secs)  &   273.94 (secs) \\
\bottomrule
\end{tabular}}
\caption{Computational performance comparison between the Hybrid Diffusion framework and Monte Carlo methods. Training times for the Hybrid Diffusion model represent a one-time cost using 50K training samples, while evaluation times show the runtime required to generate different numbers of trajectories up to the terminal time $T_{\text{max}} = 30$. Monte Carlo simulations were conducted with a step size of 5e-03.}
\label{table:time}
\end{table}



















\section{Conclusion}
\label{sec:conclusion}
This work introduces a unified hybrid data-driven framework for learning stochastic flow maps of stochastic differential equations in bounded domains, where particles can exit the computational region. The key innovation lies in addressing the fundamental challenge of particle escape through decomposing the complex bounded domain problem into two manageable components: an escape prediction neural network that learns exit probability functions for boundary phenomena, and a training-free conditional diffusion model for interior dynamics.

Our methodology offers several distinct advantages for bounded domain SDE problems that existing approaches cannot adequately handle. The escape prediction component provides a principled probabilistic approach to modeling first exit times, with rigorous convergence analysis demonstrating that the network output converges to the true exit probability. By separating exit detection from state propagation, we enable specialized handling of boundary conditions while preserving physical accuracy and reducing computational cost. The conditional diffusion model for interior dynamics eliminates the computational overhead associated with neural network-based score function learning by using a training-free approach with closed-form exact score functions.

We demonstrate the framework's effectiveness through comprehensive validation across three test cases of increasing complexity. The one-dimensional analytical validation confirms theoretical convergence properties and accuracy against exact solutions. The two-dimensional stochastic transport problem validates the framework's capability in handling mixed boundary conditions and complex flow structures with chaotic advection. The three-dimensional runaway electron application in tokamak plasmas extends traditional 2D momentum-pitch modeling to include radial transport, demonstrating the framework's practical utility for high-impact physics problems. The surrogate model accurately reconstructs complete particle distribution functions without repeated numerical integration and shows strong agreement with high-fidelity Monte Carlo simulations across multiple metrics including escape probability prediction, phase-space distribution reconstruction, and runaway electron generation under varying initial conditions. The framework achieves approximately 700$\times$ computational speedup while maintaining high accuracy and demonstrates robust generalization beyond training windows.

Future work will focus on extending this framework to simulation datasets generated by plasma kinetic transport codes such as full-orbit KORC \cite{carbajal2017space} and guiding center TAPAS \cite{zarzoso2022transport}, enabling high-fidelity 5D plasma transport modeling. Additionally, we will investigate multiscale stochastic dynamics in bounded domains, where multiple time and length scales present significant challenges for both exit probability prediction and interior dynamics modeling. We also plan to extend the methodology to bounded domain SDEs with different transport mechanisms, such as jump processes and anomalous diffusion, which exhibit fundamentally different stochastic behaviors and require specialized handling of non-Gaussian noise and memory effects.

 \section*{Acknowledgement}

This material is based upon work supported by the U.S. Department of Energy, Office of Science, Office of Advanced Scientific Computing Research, Applied Mathematics program under the contract ERKJ443, Office of Fusion Energy Science, and Scientific Discovery through Advanced Computing (SciDAC) program, at the Oak Ridge National Laboratory, which is operated by UT-Battelle, LLC, for the U.S. Department of Energy under Contract DE-AC05-00OR22725. DdCN acknowledges support from the U.S. Department of Energy under Contract No. DE-FG02-04ER-54742.

\bibliographystyle{abbrvnat}  
\bibliography{refs}

\begin{thebibliography}{30}
\providecommand{\natexlab}[1]{#1}
\providecommand{\url}[1]{\texttt{#1}}
\expandafter\ifx\csname urlstyle\endcsname\relax
  \providecommand{\doi}[1]{doi: #1}\else
  \providecommand{\doi}{doi: \begingroup \urlstyle{rm}\Url}\fi

\bibitem[Archambeau et~al.(2007)Archambeau, Cornford, Opper, and
  Shawe-Taylor]{archambeau2007gaussian}
C.~Archambeau, D.~Cornford, M.~Opper, and J.~Shawe-Taylor.
\newblock Gaussian process approximations of stochastic differential equations.
\newblock In \emph{Gaussian Processes in Practice}, pages 1--16. PMLR, 2007.

\bibitem[Carbajal et~al.(2017)Carbajal, del Castillo-Negrete, Spong, Seal, and
  Baylor]{carbajal2017space}
L.~Carbajal, D.~del Castillo-Negrete, D.~Spong, S.~Seal, and L.~Baylor.
\newblock Space dependent, full orbit effects on runaway electron dynamics in
  tokamak plasmas.
\newblock \emph{Physics of Plasmas}, 24\penalty0 (4), 2017.

\bibitem[Chen et~al.(2021)Chen, Yang, Duan, and Karniadakis]{chen2021solving}
X.~Chen, L.~Yang, J.~Duan, and G.~E. Karniadakis.
\newblock Solving inverse stochastic problems from discrete particle
  observations using the fokker--planck equation and physics-informed neural
  networks.
\newblock \emph{SIAM Journal on Scientific Computing}, 43\penalty0
  (3):\penalty0 B811--B830, 2021.

\bibitem[Chen and Xiu(2024{\natexlab{a}})]{chen2024data}
Y.~Chen and D.~Xiu.
\newblock Data-driven effective modeling of multiscale stochastic dynamical
  systems.
\newblock \emph{arXiv preprint arXiv:2408.14821}, 2024{\natexlab{a}}.

\bibitem[Chen and Xiu(2024{\natexlab{b}})]{chen2024learning}
Y.~Chen and D.~Xiu.
\newblock Learning stochastic dynamical system via flow map operator.
\newblock \emph{Journal of Computational Physics}, 508:\penalty0 112984,
  2024{\natexlab{b}}.

\bibitem[Chen and Xiu(2024{\natexlab{c}})]{chen2024modeling}
Y.~Chen and D.~Xiu.
\newblock Modeling unknown stochastic dynamical system subject to external
  excitation.
\newblock \emph{arXiv preprint arXiv:2406.15747}, 2024{\natexlab{c}}.

\bibitem[Friedrich et~al.(2011)Friedrich, Peinke, Sahimi, and
  Tabar]{friedrich2011approaching}
R.~Friedrich, J.~Peinke, M.~Sahimi, and M.~R.~R. Tabar.
\newblock Approaching complexity by stochastic methods: From biological systems
  to turbulence.
\newblock \emph{Physics Reports}, 506\penalty0 (5):\penalty0 87--162, 2011.

\bibitem[Higham(2001)]{higham2001algorithmic}
D.~J. Higham.
\newblock An algorithmic introduction to numerical simulation of stochastic
  differential equations.
\newblock \emph{SIAM review}, 43\penalty0 (3):\penalty0 525--546, 2001.

\bibitem[Higham et~al.(2013)Higham, Mao, Roj, Song, and Yin]{higham2013mean}
D.~J. Higham, X.~Mao, M.~Roj, Q.~Song, and G.~Yin.
\newblock Mean exit times and the multilevel monte carlo method.
\newblock \emph{SIAM/ASA Journal on Uncertainty Quantification}, 1\penalty0
  (1):\penalty0 2--18, 2013.

\bibitem[Kobyzev et~al.(2020)Kobyzev, Prince, and
  Brubaker]{kobyzev2020normalizing}
I.~Kobyzev, S.~J. Prince, and M.~A. Brubaker.
\newblock Normalizing flows: An introduction and review of current methods.
\newblock \emph{IEEE transactions on pattern analysis and machine
  intelligence}, 43\penalty0 (11):\penalty0 3964--3979, 2020.

\bibitem[Kolobov(2003)]{kolobov2003fokker}
V.~I. Kolobov.
\newblock Fokker--planck modeling of electron kinetics in plasmas and
  semiconductors.
\newblock \emph{Computational materials science}, 28\penalty0 (2):\penalty0
  302--320, 2003.

\bibitem[Li et~al.(2024)Li, Zhao, Xu, Duan, and Liu]{li2024deep}
Y.~Li, F.~Zhao, S.~Xu, J.~Duan, and X.~Liu.
\newblock A deep learning method for computing mean exit time excited by weak
  gaussian noise.
\newblock \emph{Nonlinear Dynamics}, 112\penalty0 (7):\penalty0 5541--5554,
  2024.

\bibitem[Liu et~al.(2024{\natexlab{a}})Liu, Chen, Xiu, and
  Zhang]{liu2024training}
Y.~Liu, Y.~Chen, D.~Xiu, and G.~Zhang.
\newblock A training-free conditional diffusion model for learning stochastic
  dynamical systems.
\newblock \emph{arXiv preprint arXiv:2410.03108}, 2024{\natexlab{a}}.

\bibitem[Liu et~al.(2024{\natexlab{b}})Liu, Yang, Zhang, Bao, Cao, and
  Zhang]{liu2024diffusion}
Y.~Liu, M.~Yang, Z.~Zhang, F.~Bao, Y.~Cao, and G.~Zhang.
\newblock Diffusion-model-assisted supervised learning of generative models for
  density estimation.
\newblock \emph{Journal of Machine Learning for Modeling and Computing},
  5\penalty0 (1), 2024{\natexlab{b}}.

\bibitem[McEneaney(2007)]{mceneaney2007curse}
W.~M. McEneaney.
\newblock A curse-of-dimensionality-free numerical method for solution of
  certain hjb pdes.
\newblock \emph{SIAM journal on Control and Optimization}, 46\penalty0
  (4):\penalty0 1239--1276, 2007.

\bibitem[Nguyen et~al.(2019)Nguyen, Simsekli, Gurbuzbalaban, and
  Richard]{nguyen2019first}
T.~H. Nguyen, U.~Simsekli, M.~Gurbuzbalaban, and G.~Richard.
\newblock First exit time analysis of stochastic gradient descent under
  heavy-tailed gradient noise.
\newblock \emph{Advances in neural information processing systems}, 32, 2019.

\bibitem[Nichol and Dhariwal(2021)]{nichol2021improved}
A.~Q. Nichol and P.~Dhariwal.
\newblock Improved denoising diffusion probabilistic models.
\newblock In \emph{International conference on machine learning}, pages
  8162--8171. PMLR, 2021.

\bibitem[Opper(2019)]{opper2019variational}
M.~Opper.
\newblock Variational inference for stochastic differential equations.
\newblock \emph{Annalen der Physik}, 531\penalty0 (3):\penalty0 1800233, 2019.

\bibitem[Peeters and Strintzi(2008)]{peeters2008fokker}
A.~G. Peeters and D.~Strintzi.
\newblock The fokker-planck equation, and its application in plasma physics.
\newblock \emph{Annalen der Physik}, 520\penalty0 (2-3):\penalty0 142--157,
  2008.

\bibitem[Pinkus(1999)]{pinkus1999approximation}
A.~Pinkus.
\newblock Approximation theory of the mlp model in neural networks.
\newblock \emph{Acta numerica}, 8:\penalty0 143--195, 1999.

\bibitem[Qi and Harlim(2023)]{qi2023data}
D.~Qi and J.~Harlim.
\newblock A data-driven statistical-stochastic surrogate modeling strategy for
  complex nonlinear non-stationary dynamics.
\newblock \emph{Journal of Computational Physics}, 485:\penalty0 112085, 2023.

\bibitem[Sobczyk(2013)]{sobczyk2013stochastic}
K.~Sobczyk.
\newblock \emph{Stochastic differential equations: with applications to physics
  and engineering}, volume~40.
\newblock Springer Science \& Business Media, 2013.

\bibitem[Song et~al.(2020)Song, Sohl-Dickstein, Kingma, Kumar, Ermon, and
  Poole]{song2020score}
Y.~Song, J.~Sohl-Dickstein, D.~P. Kingma, A.~Kumar, S.~Ermon, and B.~Poole.
\newblock Score-based generative modeling through stochastic differential
  equations.
\newblock \emph{arXiv preprint arXiv:2011.13456}, 2020.

\bibitem[Talkner et~al.(1987)Talkner, H{\"a}nggi, Freidkin, and
  Trautmann]{talkner1987discrete}
P.~Talkner, P.~H{\"a}nggi, E.~Freidkin, and D.~Trautmann.
\newblock Discrete dynamics and metastability: Mean first passage times and
  escape rates.
\newblock \emph{Journal of Statistical Physics}, 48:\penalty0 231--254, 1987.

\bibitem[Yang et~al.(2020)Yang, Zhang, and Karniadakis]{yang2020physics}
L.~Yang, D.~Zhang, and G.~E. Karniadakis.
\newblock Physics-informed generative adversarial networks for stochastic
  differential equations.
\newblock \emph{SIAM Journal on Scientific Computing}, 42\penalty0
  (1):\penalty0 A292--A317, 2020.

\bibitem[Yang et~al.(2021)Yang, Zhang, del Castillo-Negrete, and
  Stoyanov]{yang2021feynman}
M.~Yang, G.~Zhang, D.~del Castillo-Negrete, and M.~Stoyanov.
\newblock A feynman-kac based numerical method for the exit time probability of
  a class of transport problems.
\newblock \emph{Journal of Computational Physics}, 444:\penalty0 110564, 2021.

\bibitem[Yang et~al.(2023)Yang, del Castillo-Negrete, Cao, and
  Zhang]{yang2023probabilistic}
M.~Yang, D.~del Castillo-Negrete, Y.~Cao, and G.~Zhang.
\newblock A probabilistic scheme for semilinear nonlocal diffusion equations
  with volume constraints.
\newblock \emph{SIAM Journal on Numerical Analysis}, 61\penalty0 (6):\penalty0
  2718--2743, 2023.

\bibitem[Yang et~al.(2024)Yang, Wang, del Castillo-Negrete, Cao, and
  Zhang]{yang2024pseudoreversible}
M.~Yang, P.~Wang, D.~del Castillo-Negrete, Y.~Cao, and G.~Zhang.
\newblock A pseudoreversible normalizing flow for stochastic dynamical systems
  with various initial distributions.
\newblock \emph{SIAM Journal on Scientific Computing}, 46\penalty0
  (4):\penalty0 C508--C533, 2024.

\bibitem[Zarzoso et~al.(2022)Zarzoso, del Castillo-Negrete, Lacroix, Bernard,
  and Touzet]{zarzoso2022transport}
D.~Zarzoso, D.~del Castillo-Negrete, R.~Lacroix, P.-E. Bernard, and S.~Touzet.
\newblock Transport and losses of fusion-born alpha particles in the presence
  of tearing modes using the new toroidal accelerated particle simulator
  (tapas).
\newblock \emph{Plasma Physics and Controlled Fusion}, 64\penalty0
  (4):\penalty0 044003, 2022.

\bibitem[Zhang and del Castillo-Negrete(2017)]{zhang2017backward}
G.~Zhang and D.~del Castillo-Negrete.
\newblock A backward monte-carlo method for time-dependent runaway electron
  simulations.
\newblock \emph{Physics of Plasmas}, 24\penalty0 (9), 2017.

\end{thebibliography}
\end{document}